\documentclass[onefignum,onetabnum]{siamonline220329}

\usepackage{enumitem}
\usepackage{cite}
\usepackage{wrapfig}
\usepackage{amsmath,bm}
\usepackage{lipsum}
\usepackage{amsfonts}
\usepackage{graphicx}
\usepackage{epstopdf}
\usepackage{adjustbox}
\usepackage{algorithmic}

\ifpdf
  \DeclareGraphicsExtensions{.eps,.pdf,.png,.jpg}
\else
  \DeclareGraphicsExtensions{.eps}
\fi

\usepackage{enumitem}
\setlist[enumerate]{leftmargin=.5in}
\setlist[itemize]{leftmargin=.5in}

\newsiamremark{remark}{Remark}
\newsiamremark{hypothesis}{Hypothesis}
\crefname{hypothesis}{Hypothesis}{Hypotheses}
\newsiamthm{claim}{Claim}
\newsiamthm{assum}{Assumption}
\newsiamthm{scheme}{Scheme}

\headers{Conditional Pseudo-Reversible Normalizing Flow for UQ}{M.~Yang, P.~Wang, M.~Fan, D.~Lu, Y.~Cao, G.~Zhang}

\title{Conditional Pseudo-Reversible Normalizing Flow for Surrogate Modeling in Quantifying Uncertainty Propagation
\thanks{
{This manuscript has been authored by UT-Battelle, LLC, under contract DE-AC05-00OR22725 with the US Department of Energy (DOE). The US government retains and the publisher, by accepting the article for publication, acknowledges that the US government retains a nonexclusive, paid-up, irrevocable, worldwide license to reproduce the published form of this manuscript, or allow others to do so, for US government purposes. DOE will provide public access to these results of federally sponsored research in accordance with the DOE Public Access Plan.}}}

\author{Minglei Yang\thanks{Fusion Enery Science Division, Oak Ridge National Laboratory, Oak Ridge, TN 37831.}
\and 
Pengjun Wang\thanks{Department of Mathematics and Statistics, Auburn University, Auburn, AL 36849.}
\and 
Ming Fan\thanks{Computational Sciences and Engineering Division, Oak Ridge National Laboratory, Oak Ridge, TN 37831.}
\and 
Dan Lu\footnotemark[4]
\and
Yanzhao Cao\footnotemark[3]
\and
Guannan Zhang\thanks{Corresponding author. Computer Science and Mathematics Division, Oak Ridge National Laboratory, Oak Ridge, TN 37831. (zhangg@ornl.gov)}
}

\usepackage{amsopn}

\begin{document}

\maketitle

\begin{abstract}

We introduce a conditional pseudo-reversible normalizing flow for constructing surrogate models of a physical model polluted by additive noise to efficiently quantify forward and inverse uncertainty propagation. Existing surrogate modeling approaches  usually focus on approximating the deterministic component of physical model. However, this strategy necessitates knowledge of noise and resorts to auxiliary sampling methods for quantifying inverse uncertainty propagation. In this work, we develop the conditional pseudo-reversible normalizing flow model to directly learn and efficiently generate samples from the conditional probability density functions. The training process utilizes dataset consisting of input-output pairs without requiring prior knowledge about the noise and the function. Our model, once trained, can generate samples from any conditional probability density functions whose high probability regions are covered by the training set. Moreover, the pseudo-reversibility feature allows for the use of fully-connected neural network architectures, which simplifies the implementation and enables theoretical analysis. We provide a rigorous convergence analysis of the conditional pseudo-reversible normalizing flow model, showing its ability to converge to the target conditional probability density function using the Kullback–Leibler divergence. To demonstrate the effectiveness of our method, we apply it to several benchmark tests and a real-world geologic carbon storage problem.

\end{abstract}









\section{Introduction}\label{sec:intro}

Uncertainty quantification (UQ) involves the process of identifying and understanding the uncertainties inherent in physical models, numerical simulations, or experimental data. It is extensively applied in fields such as material science, physics, and other areas characterized by complex scientific models. The evolution of mathematical methodologies and computational power has profoundly enhanced our grasp of the dynamics within physical models and the propagation of uncertainties. While these advanced techniques are typically effective in small and straightforward models, their accuracy may decline when applied to complex physical models with numerous parameters, large datasets, or insufficient input data, highlighting the necessity for continued exploration in uncertainty estimation.


The critical role of uncertainty quantification necessitates efficient computational models, making the construction of surrogate models for conditional probability density functions  a key strategy for achieving computational efficiency in uncertainty analysis. Existing surrogate modeling techniques, e.g., polynomial chaos expansions \cite{kersaudy2015new}, kernel methods \cite{wirtz2015surrogate}, and Gaussian process regression \cite{zhou2005study}, are adept at approximating deterministic model components. However, they often inadequately capture the uncertainty introduced by 
noise without the knowledge of its distribution or resorting to auxiliary sampling methods for quantifying inverse uncertainty propagation. This underscores the need for advanced surrogate models that integrate both deterministic and stochastic aspects of models, thereby enhancing the comprehensiveness and efficiency of UQ efforts.

Deep learning approaches have gained significant attention for uncertainty estimation, handling complex datasets and high-dimensional challenges effectively \cite{abdar2021review,psaros2023uncertainty}. Techniques such as variational inference \cite{blei2017variational}, Bayesian neural networks \cite{wilson2020bayesian,jospin2022hands}, rooted in Bayesian models, offer a rigorous framework for managing model uncertainty. However, these methods face computational difficulties, including high parameter counts and slow convergence, which are deteriorated when involving additional physical constraints. Solutions like MC dropout \cite{gal2016dropout} and stochastic variational inference \cite{hoffman2013stochastic} mitigate some computational issues. Moreover, Bayesian inference is highly influenced by prior information and has to retrain for the new data. To tackle this challenge, conditional distributions, e.g., amortized variational distributions \cite{padmanabha2021solving,karumuri2023learning}, are employed as a strategy to evaluate the distribution at new sample points.

Another type of deep learning approaches is to construct surrogate models utilizing generative models, e.g., generative adversarial networks (GANs) \cite{goodfellow2014generative}, variational autoencoders (VAEs) \cite{kingma2013auto}, and diffusion models \cite{yang2022diffusion}, to tackle uncertainty estimation \cite{sensoy2020uncertainty, ratzlaff2019hypergan, bohm2019uncertainty}. Generative models can offer an optimal prior for data knowledge and seek the maximum of the posterior distribution.
Normalizing flow, as a generative model, has the capability to represent complex posterior distributions. The normalizing field flows (NFF) model \cite{guo2022normalizing} is introduced to measure uncertainties arising from both forward and inverse stochastic partial differential equations. Ref \cite{selvan2020uncertainty} introduced a conditional normalizing flow model to capture medical image segmentation by extending the expressivity of conditional variational autoencoders \cite{zhao2017learning}.
However, it is important to note that while normalizing flows have been successfully applied in uncertainty estimation, they rely on invertible bijections, which imposes unfavorable constraints on their scalability and the range of distributions they can effectively learn.


In this work, we introduce a conditional pseudo-reversible normalizing flow (PR-NF) as a surrogate model to effectively quantify uncertainty propagation. Our goal is to accurately determine the conditional distributions for physical models in both forward  and inverse directions. We demonstrate that the PR-NF model can leverage the same dataset and neural network architecture to characterize uncertainties in both directions, illustrating the model's robustness against variations in the input-output function relationship.
The PR-NF architecture utilizes a straightforward feed-forward neural network to simulate flows, integrating an additional loss term to ensure the reversibility. This innovative neural network structure offers enhanced flexibility for flow transformations, potentially boosting model performance. Notably, the adoption of GPU-accelerated  libraries markedly accelerates the computation of the Jacobian determinant and PyTorch backpropagation, reducing computational costs.
We conduct a comprehensive convergence analysis for the conditional PR-NF model, specifically focusing on its convergence towards the target conditional distribution of uncertainty, a metric quantified using the Kullback–Leibler divergence. Once trained, the PR-NF model efficiently produces accurate outputs for new samples within the domain of the training dataset, eliminating the need for retraining. This significantly improves the efficiency of uncertainty estimation and sampling. However, it is important to note that the PR-NF model is designed to operate within its training domain and does not inherently possess the capability to predict beyond this initial domain.

The rest of this paper is structured as follows:
In Section \ref{sec:problem}, we define the surrogate modeling problem of interest.
The details about the conditional pseudo-reversible normalizing flow model are discussed in Section \ref{sec:method}.
Section \ref{sec:analy} provides the convergence analysis of PR-NF model on uncertainty estimation.
Finally, in Section \ref{sec:example}, we present various examples, including a  validation model with different types of additive noise, high-dimensional uncertainty estimation problems, and an application in earth system science. The conclusion is in Section \ref{sec:con}.

%
\section{Problem setting}\label{sec:problem}
This section introduces the surrogate modeling problem under consideration.
We consider a physical model of the form
\begin{equation}\label{eq:problem}
      {\bm Y} = {\bm f}({\bm X}) + {\bm \varepsilon}({\bm X}),
\end{equation}
where $\bm X \in \mathbb{R}^d$ is a $d$-dimensional random vector denoting the input of the model, ${\bm f}:\mathbb{R}^{d} \rightarrow \mathbb{R}^s$ is a continuously differentiable function that represents the deterministic component of the model, ${\bm Y} \in \mathbb{R}^s$ is a $s$-dimensional random vector denoting the output of the model, and $\bm \varepsilon(\bm X)$ denotes the additive random noise satisfying 
$
\mathbb{E}[{\bm \varepsilon}({\bm X})] = 0$ and $\mathbb{E}[{\bm \varepsilon^2}({\bm X})] = {\bm v}({\bm X}) 
$
with ${\bm v}(\cdot)$ being a deterministic and bounded function. 
The goal is to efficiently compute two types of quantities of interest (QoIs), i.e., 
\begin{equation}\label{eq:forward}
    \text{(Forward QoI)}\quad Q_{\rm forward}(\bm x) = \int_{\mathbb{R}^s} q_{\rm forward}(\bm y)\; p( \bm y | \bm x)\, d\bm y\; \text{ for }\;\bm X = \bm x,
\end{equation}
where $q_{\rm forward}(\bm y)$ is a function of the model's output and $p( \bm y | \bm x)$ is the conditional probability density function (PDF) describing how the uncertainty of $\bm \varepsilon(\bm x)$ is propagated to $\bm Y$ for $\bm X = \bm x$, and 
\begin{equation}\label{eq:inverse}
    \text{(Inverse QoI)}\quad Q_{\rm inverse}(\bm y) = \int_{\mathbb{R}^d} q_{\rm inverse}(\bm x)\; p( \bm x | \bm y)\, d\bm x\; \text{ for }\;\bm Y = \bm y, 
\end{equation}
where $q_{\rm inverse}(\bm x)$ is a function of the model's input  and $p(\bm x |\bm y)$ is the conditional PDF of $\bm X$ for any fixed value of the model's output $\bm Y = \bm y$. 

In practice, it is common to define a prior distribution of $\bm X$ to be one of the classical distributions, e.g., uniform or normal distributions, based on knowledge of the physical problem under consideration. Thus, we assume that the input $\bm X$ follows a prior PDF, denoted by 
\begin{equation}\label{eq:prior}
 \bm X \sim p(\bm x),
\end{equation}
and we are capable of efficiently generating unlimited number of samples of $\bm X$ from the prior $p(\bm x)$. The definitions of $p(\bm x)$, $\bm f(\cdot)$ and $\bm\varepsilon (\cdot)$ can uniquely determine the joint distribution $p(\bm x, \bm y)$ and the marginal distribution $p(\bm y) = \int p(\bm x,\bm y) d\bm x$, such that the conditional PDF $p(\bm x|\bm y)$ in Eq.~\eqref{eq:inverse} is defined by 
\begin{equation}\label{eq:joint}
    p(\bm x | \bm y) = \frac{p(\bm x, \bm y)}{p(\bm y)},
\end{equation}
based on the conditional probability formula. Note that Eq.~\eqref{eq:joint} is only the definition of the condition PDF $p(\bm x|\bm y)$, but we do not know how to draw samples from $p(\bm x|\bm y)$. 

\subsection{Challenges of building surrogate models for the conditional PDFs}\label{sec:challenge}
When the physical model $\bm f(\bm X)$ is computationally expensive to evaluate, computing the QoIs in Eqs.~\eqref{eq:forward} and \eqref{eq:inverse} may become computationally infeasible. In this scenario, building a surrogate model to replace the time-consuming physics model is a widely used strategy to significantly improve the computational efficiency in computing the QoIs. Existing surrogate modeling approaches, e.g., polynomial chaos, kernel methods, or neural networks, mainly focus approximating the deterministic function $\bm f(\bm X)$ in Eq.~\eqref{eq:problem}. 
However, the strategy of approximating $\bm f(\bm X)$ has the following disadvantages:
\vspace{0.1cm}
\begin{itemize}[leftmargin=20pt]\itemsep0.15cm
    \item It requires capability of generating samples of the noise $\bm \varepsilon(\bm X)$ in order to obtain samples of the conditional PDF $p(\bm y |\bm x)$. When we do not know the distribution of $\bm \varepsilon(\bm X)$, we cannot use the deterministic surrogate of $\bm f(\bm X)$ to compute the forward QoI in Eq.~\eqref{eq:forward}.
    \item It requires another sampling method, e.g., Markov Chain Monte Carlo (MCMC), to generate samples from the conditional PDF $p(\bm x| \bm y)$ to compute the inverse QoI in Eq.~\eqref{eq:inverse}. 
\end{itemize}
\vspace{0.1cm}
To address these challenges, we intend to utilize a normalizing flow based generative model to directly learn how to generate samples from the conditional PDFs $p(\bm y|\bm x)$ in Eq.~\eqref{eq:forward} and 
$p(\bm x | \bm y)$ in Eq.~\eqref{eq:inverse}. 
The trained normalizing flow models will serve as the surrogate models for the conditional PDFs, such that we can generate unlimited samples of the conditional PDFs to efficiently compute the QoIs in Eqs.~\eqref{eq:forward} and \eqref{eq:inverse}.



\section{The pseudo-reversible normalizing flow (PR-NF) for learning conditional PDFs}\label{sec:method}
In this section, we explain in detail the proposed conditional generative model for computing the QoIs. To proceed, we define the training dataset as 
\begin{equation}\label{eq:train_set}
    \mathcal{D}_{\rm train} := \left\{\bm x^{(n)}, \bm y^{(n)} \right\}_{n=1}^N
    \;\; \text{with} \;\; \bm y^{(n)} = \bm f\left(\bm x^{(n)}\right) + \bm \varepsilon^{(n)},
\end{equation}
where $\bm y^{(n)}$ is the output of the physical model in Eq.~\eqref{eq:problem} for the input $\bm x^{(n)}$. Hereafter, we assume a purely data-driven scenario, which means we will need to train a surrogate model of the conditional PDFs of interest without using any information about the model $\bm f$ and the noise $\bm \varepsilon$. Taking $p(\bm y | \bm x)$ as an example, our objective is to build and train a conditional generative model, denoted by 
\begin{equation}\label{eq:gen}
    \bm Y| \bm X \approx \widehat{\bm Y} = \bm G(\bm Z, \bm X; \bm \theta_{\bm G}),
\end{equation}
where $\bm Z$ follows the standard normal distribution $\mathcal{N}(0, \mathbf{I}_p)$, $\bm X$ is the input of the physics model in Eq.~\eqref{eq:problem}, $\bm \theta_{\bm G}$ denotes the trainable parameters of $\bm G$. After training, the output of the generator approximates the target conditional random variable $\bm Y |\bm X$, such that the forward QoI in Eq.\eqref{eq:forward} can be efficiently computed by drawing samples from $\mathcal{N}(0, \mathbf{I}_p)$ and pushing through the generator $\bm G$. A similar generator can be constructed for the conditional PDF $p(\bm x|\bm y)$. In this effort, we construct the conditional generative model in Eq.~\eqref{eq:gen} by the PR-NF architecture introduced in the next subsection.

\subsection{The architecture of conditional PR-NF model}\label{sec:PRNF}
The standard pseudo-reversible architecture is similar to auto-encoders except that the bottle neck layer has the same width as the input and output layers. Compared to exactly reversible neural networks, the pseudo-reversible architecture has two advantages. First, it allows the use of simple fully-connected networks which greatly simplifies the implementation of normalizing flow models. Second, it allows us to conduct rigorous convergence analysis of normalizing flow models, which has a great significance from the mathematics perspective. In this section, we will explain how to extend the standard NF to the conditional PR-NF. 


For notational simplicity, we introduce a new random vector $\bm W$ defined by 
\begin{equation}\label{eq:simple}
    \bm W:=(\bm X, \bm Y|\bm X) \in \mathbb{R}^{d+s},
\end{equation}
which is the concatenation of $\bm X$ and $\bm Y|\bm X$.
%
A normalizing flow model includes an encoding transport map $\bm h$ and an decoding transport map $\bm g$, i.e.,
\begin{equation}\label{eq:NF}
    \bm Z =  \bm h(\bm W; \bm \theta_{\bm h})\;\; \text{ and }\;\; \widehat{\bm W} =  \bm g(\bm Z; \bm \theta_{\bm g}),
\end{equation}
where $\widehat{\bm W}$ is the approximation of $\bm W$ provided by the normalizing flow model, and $\bm \theta_{\bm h}$, $\bm \theta_{\bm g}$ are trainable parameters of the transformations. In standard normalizing flow models, e.g., MAF\cite{papamakarios2017masked} and Real NVPs\cite{dinh2016density}, the transport maps $\bm h$ and $\bm g$ are usually defined by exact reversible neural networks, such that $\bm g = \bm h^{-1}$ and $\widehat{\bm W} = \bm W$. In practice, the enforcement of reversibility could limit the expressive power and make it difficult to conduct any theoretical analysis. In PR-NF, 
$\bm h$ and $\bm g$ are defined by two independent fully-connected neural networks and the pseudo-reversibility is enforced by incorporating a soft constraint $\|\widehat{\bm W} - \bm W\|_2^2$ into the loss function introduced in Section \ref{sec:loss}.
Let $p_{\bm W}(\bm w)$ and $p_{\bm Z}(\bm z)$ represent the PDFs of variables $\bm W$ and $\bm Z$, respectively. The connection between these two PDFs can be established using the change of variables formula, i.e., 
\begin{equation}
\begin{aligned}
p_{\bm W}({\bm w}) = p_{\bm Z}({\bm z}) \left| \frac{\partial {\bm z}}{\partial {\bm w}} \right| &= p_{\bm Z}({\bm h}({\bm w}))
|{\rm det} {\mathbf{J}_{\bm h}}({\bm w})|,
\end{aligned}
\end{equation}
where ${\rm det} {\mathbf{J}_{\bm h}}({\bm w})$ is the determinant of the Jacobian matrix $\mathbf{J}_{\bm h}$ of the encoding map $\bm h(\cdot)$.
\begin{figure}[h!]
    \centering
  {\includegraphics[width=0.7\textwidth]{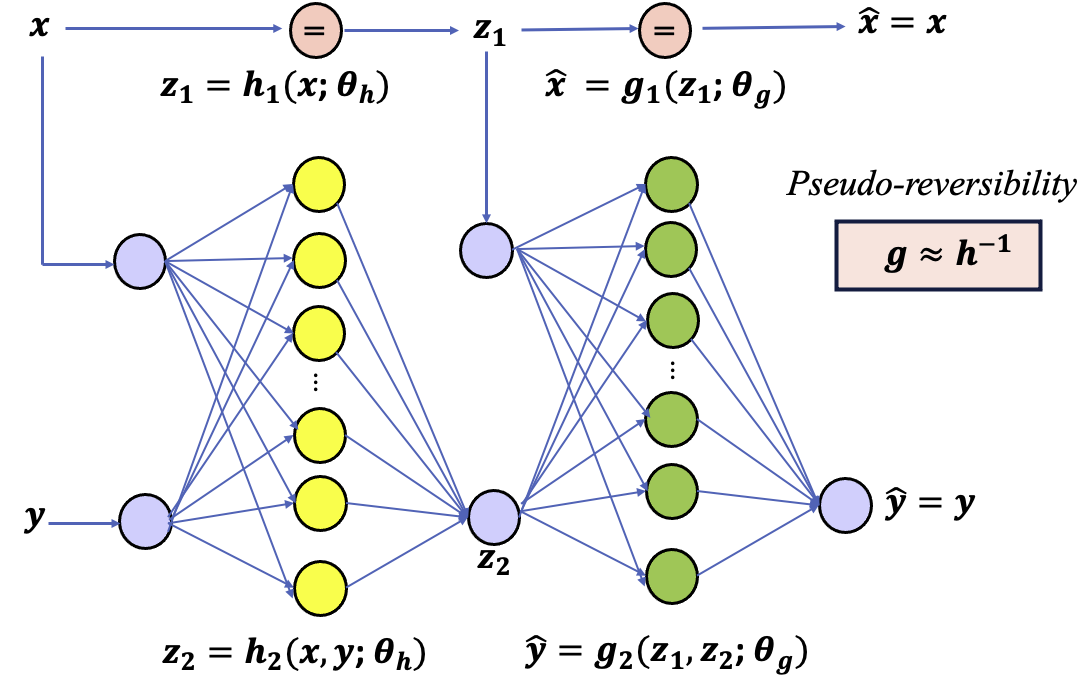}}
  \vspace{-0.3cm}
    \caption{The architecture of the proposed conditional pseudo-reversible normalizing flow (PR-NF) model. The novelty is that the input vector $\bm x$ of the physics model in Eq.~\eqref{eq:problem} is incorporated into the hidden neurons of $\bm h(\cdot)$ and $\bm g(\cdot)$. In other words, both $\bm h(\cdot)$ and $\bm g(\cdot)$ are parameterized by $\bm x$, which is the main reason why the transport maps can learn the conditional PDF $p(\bm y|\bm x)$. The pseudo-reversibility $\bm g \approx {\bm h}^{-1}$ is enforced as a soft constraint in the loss function. }
    \label{fig_nf_x0}
    \vspace{-0.3cm}
\end{figure}

The architecture of the conditional PR-NF neural network is illustrated in Figure \ref{fig_nf_x0}. Specifically, the transport map $\bm h(\cdot)$ consists of two components: $\bm h(\bm x, \bm y|\bm x) = (\bm h_1(\bm x), \bm h_2(\bm x, \bm y|\bm x; \bm \theta_h))$. Here, $\bm h_2(\bm x, \bm y|\bm x; \bm \theta_{\bm h})$ represents a fully connected neural network with parameter $\bm \theta_{\bm h}$, mapping from $\mathbb{R}^{d+s}$ to $\mathbb{R}^s$. On the other hand, $\bm h_1(\bm x)$ serves as an identity mapping, dependent on $\bm x$.
To describe the conditional relationship between $\bm x$ and $\bm y$ within the conditional distribution $p(\bm y|\bm x)$, both are included as input features in the neural network. This allows the neural network to learn the distribution of one input variable, $\bm y$, conditioned on the other, $\bm x$.
We denote the output of $\bm h_1(\bm x)$ as $\bm z_1$ and the output of $\bm h_2(\bm x, \bm y|\bm x; \bm \theta_{\bm h})$ as $\bm z_2$.
The inverse map $\bm g$ exhibits a similar structure to the forward map, expressed as $\bm g(\bm z_1, \bm z_2) = (\bm g_1(\bm z_1), \bm g_2(\bm z_1, \bm z_2; \bm \theta_g))$. Here, $\bm g_1$ serves as an identity map, while $\bm g_2$ is a fully connected neural network with parameter $\bm \theta_{\bm g}$. The outputs of $\bm g_1$ and $\bm g_2$ are denoted as $\widehat{\bm x}$ and $\widehat{\bm y}$, respectively. It's worth noting that $\widehat{\bm x} = \bm z_1 = \bm x$ because both $\bm h_1$ and $\bm g_1$ are identity mappings.

\subsection{The loss function for training the conditional PR-NF}\label{sec:loss}
We can rewrite the training set in Eq.~\eqref{eq:train_set} using the simplified notation in Eq.~\eqref{eq:simple}, i.e.,
\begin{equation}\label{eq:train}
\mathcal{D}_{\rm train} = \left\{ \bm w^{(n)}\right\}_{n=1}^{N} =\left\{\bm x^{(n)}, \bm y^{(n)}| \bm x^{(n)}\right\}_{n=1}^{N},
\end{equation}
where $\bm y^{(n)}| \bm x^{(n)} = \bm y^{(n)}$ because $\bm y^{(n)}$ is obtained by push $\bm x^{(n)}$ through the physical model in Eq.~\eqref{eq:problem}.
%
%
The loss function consists of two components, i.e., 
\begin{equation}\label{eq_loss}
\mathcal{L} = \mathcal{L}_1 + \lambda \mathcal{L}_2,
\end{equation}
where $\mathcal{L}_1$ is the negative log-likelihood loss defined by
\begin{equation}\label{eq:loglike}
 \mathcal{L}_1 = - \frac{1}{N}\sum_{n=1}^{N} \left({\rm log}p_{{\bm Z}_2}({ \bm h}_2({\bm w}^{(n)}; \bm \theta_{h})) +  {\rm log} \left|{\rm det} {\bf J}_{\bm h} ({\bm w}^{(n)}; \bm \theta_{\bm h})\right|\right),
\end{equation}
with $p_{{\bm Z}_2}(\cdot)$ is the probability density function of the standard normal distribution, and $\mathcal{L}_2$ is the pseudo-reversibility loss that measures the difference between ${\bm w}$ and $\widehat{\bm w}$, i.e.,
\begin{equation}\label{eq:loss2}
    \mathcal{L}_2 = \frac{1}{N} \sum_{n=1}^{N} \left( \left\|{\bm w}^{(n)} - {\bm g} ( {\bm h}({\bm w}^{(n)} ; {\bm \theta}_{\bm h});  {\bm \theta}_{\bm g})\right\|_2^2  + \left|\det {\bf J}_{\bm g}({\bm h}({\bm w}^{(n)})) \det {\bf J}_{\bm h}({\bm w}^{(n)}) - 1 \right| \right),
\end{equation}
where the hyperparameter $\lambda$ determines and balances the importance of the reversibility loss $\mathcal{L}_2$ in the total loss $\mathcal{L}$. 

A process of hyperparameter tuning is required to achieve optimal model performance. To proceed, we conduct a grid search and employ cross-entropy as our metric to determine the optimal value for $\lambda$. To do this, we create a set of candidate values denoted as $\{\lambda_j\}_{j=1}^J$. For each $\lambda_j$, we'll train a PR-NF model represented as $(\bm h^{(j)}(\bm w), \bm g^{(j)}(\bm z))$. Subsequently, we'll generate $M$ samples of $\widehat{\bm w}$ by applying the inverse mapping, i.e.,
\begin{equation}\label{eq:sample}
\left\{ \widehat{\bm w}^{(j,m)} = \bm g^{(j)}(\bm z^{(m)}), m = 1, \ldots, M \right\}, 
\end{equation}
where $\{\bm z^{(m)}\}_{m=1}^M = \{\bm z_1^{(m)}, \bm z_2^{(m)}\}_{m=1}^M$, with $\bm z_1^{(m)}$ generated uniformly from $\mathcal{D}$ and $\bm z_2^{(m)}$ sampled from the standard normal distribution. This approach allows us to generate a sufficient number of samples without additional training data. Subsequently, we construct a kernel density estimator using the samples outlined in Eq.\eqref{eq:sample}, followed by computing the cross-entropy as follows:
\begin{equation}\label{eq:cross}
    H(\lambda_j) = - \frac{1}{N} \sum_{n=1}^N \log(p_{\rm KDE}( \bm w_t^{(n)})),
\end{equation}
where $p_{\rm KDE}$ represents the kernel density estimator, which is constructed based on the samples in Eq.\eqref{eq:sample}, while $\{\bm w_t^{(n)}\}_{n=1}^N$ is the training dataset in Eq.\eqref{eq:train}. The optimal hyperparameter $\lambda$ is determined by minimizing the cross-entropy, formally expressed as $\lambda = \text{argmin} H(\lambda_j)$. The numerical experiments in Sec \ref{sec:example} use the optimal hyperparameter $\lambda$ that minimizes the cross-entropy.

\subsection{The computational cost of training the conditional PR-NF model}\label{sec:cost}
%
The computational challenges in training normalizing flow models are associated with calculation of the Jacobian determinant. In the conditional PR-NF model, the Jacobian determinant in $\mathcal{L}_1$ from Eq.~\eqref{eq:loglike} is expressed as:
\begin{equation}\label{eq_Jac_reduce}
    \left|{\rm det} {\mathbf{J}_{\bm h}}({\bm w})\right| = \left|\text{det}\begin{pmatrix}
\mathbf{I}_d, \quad  0\\
\frac{\partial \bm z_2}{\partial \bm x}, \frac{\partial \bm z_2}{\partial \bm y}
\end{pmatrix} \right| = \left|\text{det}\left(\frac{\partial \bm z_2}{\partial \bm y}\right)\right|,
\end{equation}
where $\mathbf{I}_d$ is a $d\times d$ identity matrix. 
The simplification in Eq.~\eqref{eq_Jac_reduce} is based on the fact that ${\bm h}_1$ functions as an identity map of $\bm x$. As a result, the size of the Jacobian matrix ${\mathbf{J}_{\bm h}}$ depends only on the dimension of the output variable $\bm y$. Because the PR-NF model does not have a special architecture (e.g., MAF\cite{papamakarios2017masked} and Real NVPs\cite{dinh2016density}) that leads to explicit calculation of the Jacobian determinant, we employ the QR decomposition to compute the determinant, which exhibits a computational complexity of $\mathcal{O}(s^3)$. However, with the advancement of GPU acceleration techniques, the cost of QR decomposition becomes an acceptable approach for moderately high-dimensional problems. 

Fig.~\ref{fig_Jac} displays the computational cost (wall-clock time) of both Pytorch backpropagation and the calculation of the Jacobian determinant using 5,000 samples, comparing CUDA GPU and CPU versions.
While we maintain a constant input parameter dimension of ${\bm x}$ at $d=20$, we systematically vary the dimension of the observation $\bm y$ from $s= 10, 20,\ldots, 160$. The architecture employs a fully-connected neural network with one hidden layer and 512 hidden neurons.
It's worth noting that as the dimension $s$ increases, the GPU proves to be effective in accelerating the PR-NF model, particularly in the case of backpropagation, which benefits from the GPU acceleration techniques.
On the other hand, we realize that the computational cost will become unaffordable when the dimension $s$ is extremely high. In this scenario, dimension reduction methods, e.g., auto-encoders, can be used to identify low-dimensional manifold of the training dataset $\mathcal{D}_{\rm train}$, and the conditional PR-NF is then applied in the latent space. We demonstrate this strategy in solving the geologic carbon storage problem in Section \ref{gcs}.

\begin{figure}[h!]
    \centering
  {\includegraphics[width=0.45\textwidth]{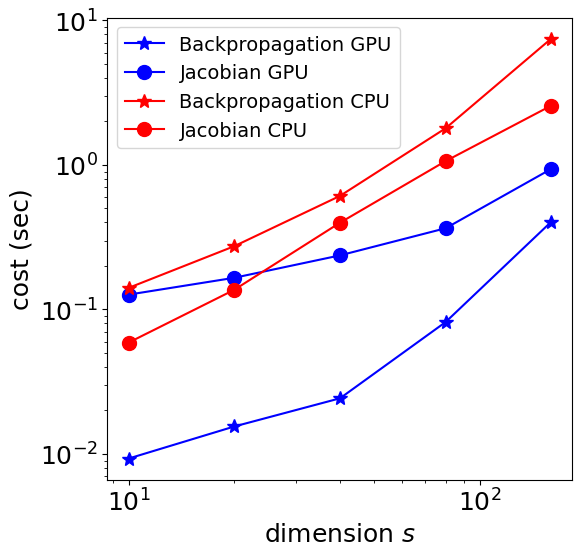}}
  \vspace{-0.2cm}
    \caption{The computational cost (wall-clock time) of both Pytorch Backpropagation and the calculation of the Jacobian determinant with 5000 samples, comparing CUDA GPU (blue) and CPU (red). The computational time is obtained by running our code on a workstation with Nvidia RTX A5000 GPU. It shows that GPU is effective in accelerating the PR-NF model, particularly in the case of backpropagation.}
    \label{fig_Jac}
    \vspace{-0.2cm}
\end{figure}

\section{Convergence analysis}\label{sec:analy}

In this section, we present an error analysis of the  proposed PR-NF model introduced in Section \ref{sec:method}.
Here we limit our attention to the forward problem ${\bm Y} = {\bm f}({\bm X}) + {\bm \varepsilon}({\bm X})$\footnote{The observation $\bm Y$ is a function in term of $\bm X$. In some places we use $\bm Y$ instead of $\bm Y|\bm X$ for notational simplicity.}.
We start by showing the convergence of the loss functions $\mathcal{L}_1$ and $\mathcal{L}_2$ in Lemma \ref{lem1}. Then, in Theorem \ref{them1}, we establish the Monte Carlo integration error related to these loss functions. Finally, in Theorem \ref{them2}, we show that with certain maid assumptions, the convergence of the loss function is sufficient to guarantee the convergence of the KL divergence between the ground truth distribution $p_{Y}$ and its approximation $p_{\widehat Y}$ obtained through the PR-NF model. 

\subsection{Convergence of the loss function in continuous form}
This section is to prove the convergence of loss functions $\mathcal{L}_1$ and $\mathcal{L}_2$ in Eq.~\eqref{eq_loss}. 
Equation 4.1 in Ref \cite{papamakarios2021normalizing} shows that for any pair of well-behaved\footnote {A well-behaved distribution $p_{\bm X}({\bm x})$ means that $p_{\bm X}({\bm x})>0$ for all ${\bm x}\in \mathbb{R}^d$, and all conditional probabilities ${\rm Pr}({\bm X}_i \le x_i | {\bm x}_{<i})$ are differentiable, for $i = 1, 2, \ldots, d$.} distributions $p_{\bm W}({\bm w})$ and $p_{\bm Z}({\bm z})$, there exists a diffeomorphism $\bm z = \bm k(\bm w) = (k_1(\bm w), \ldots, k_d(\bm w))$ that can transform $p_{\bm W}({\bm w})$ to $p_{\bm Z}({\bm z})$, i.e., 
\begin{equation}\label{eq_denp}
p_{\bm W}(\bm w) = p_{\bm Z}(\bm k(\bm w))|\det {\bf J}_{\bm k}(\bm w)|.
\end{equation}
In our analysis, we have the following Assumption \ref{ass1} for the diffeomorphism ${\bm k}$.
\begin{assum}\label{ass1}
We assume that the determinant of the Jacobian matrix and any partial derivative of $\bm k \in C^1(\mathbb{R}^{d+s})$ are bounded, i.e., 
 \begin{equation}
     \left|\frac{\partial k_i(\bm w)}{\partial \bm w_j}\right| < A_1 < \infty,
    \quad |\det \mathbf{J}_{\bm k}(\bm w)| > A_2 > 0,
 \end{equation}
where $A_1$ and $A_2$ are positive constants and $i,\ j \in \{1,\ldots,{d+s}\}$. 
Additionally, the probability density function $p_{\bm W}(\bm w)$ is bounded, all the conditional 
probabilities are differentiable, and the decay of $p_{\bm W}$ at the tail is Gaussian like, i.e., there exists a positive constant $K$ such that  
 \begin{equation}\label{assum_3}
     p_{\bm W}(\bm w) < A_3\exp(- \alpha\|\bm w\|^2)\; \text{ for }\; \|\bm w\| > K,
 \end{equation}
 where $\alpha>0$ and $A_3>0$ are positive constants and $\|\cdot\|$ denotes the $l^2$ norm.
\end{assum}

The proposed PR-NF model can be viewed as an approximation of the diffeomorphism. Specifically, two networks $\bm h(\bm w)$ and $\bm g(\bm z)$ in Eq.~\eqref{eq:NF} are approximations of  $\bm k$ and $\bm k^{-1}$, respectively. 
We first define auxiliary random variables 
\begin{equation}\label{eq:xxx}
\widetilde{{\bm W}} = {\bm h^{-1}}({\bm Z})\; \text{ and } \; \widehat{{\bm W}} = {\bm g}({\bm Z}),
\end{equation}
where ${\bm Z}$ follows the standard normal distribution. 
Using the change of variables formula, the probability density function of $\widetilde{{\bm W}}$ and $\widehat{{\bm W}}$ are defined by
\begin{equation}\label{p_xtilde}
    p_{\widetilde {\bm W}}(\bm w) = p_{\bm Z}(\bm h (\bm w))|\det {\bf J}_{\bm h}(\bm w)|, \quad p_{\widehat {\bm W}}(\bm w) = p_{\bm Z}(\bm g^{-1} (\bm w))|\det {\bf J}_{\bm g^{-1}}(\bm w)|.
\end{equation}
For simplicity, we consider the loss functions in the continuous form, i.e.,
\begin{equation}\label{con_l1}
\begin{aligned}
    \widetilde{\mathcal{L}}_1 =&-\int_{\mathcal{D}_{\bm W}} p_{\bm W}(\bm w) \left(\log p_{{\bm Z}_2}(\bm h_2(\bm w)) + \log |\det {\bf J}_{\bm h}(\bm w)| \right) d\bm w\\
    =& -\int_{\mathcal{D}_{\bm W}} p_{\bm W}(\bm w) (\log p_{\widetilde{{\bm W}}}(\bm w) - \log p_{{\bm Z}_1}(\bm z_1)) d\bm w.
\end{aligned}
\end{equation}
and 
\begin{equation}\label{con_l2}
\widetilde{\mathcal{L}}_2 = \int_{\mathcal{D}_{\bm W}} p_{\bm W}(\bm w)\left(\|{\bm g}({\bm h}(\bm w))-\bm w\|^2 + |\det {\bf J}_{\bm g}({\bm h}(\bm w)) \det {\bf J}_{\bm h}(\bm w) - 1|\right)d\bm w,
\end{equation}
where $\mathcal{D}_{\bm W}$ is the domain of variable ${\bm W}$. For variables ${\bm W}$ and $\widetilde{{\bm W}}$ we have the following Assumption \ref{ass2}.
\begin{assum}\label{ass2}
We assume  density functions $p_{\bm W}(\bm w)$ and $p_{\widetilde {\bm W}}(\bm w)$ satisfy
 \begin{equation}\label{assum_4}
     A_4\exp(-\alpha \|\bm w\|^2) < \frac{p_{\bm W}(\bm w)}{p_{\widetilde{{\bm W}}}(\bm w)} < A_4\exp(\alpha \|\bm w\|^2),
 \end{equation}  
 where $\alpha$ and $A_4$ are positive constants.
\end{assum}

Under the Assumptions \ref{ass1} and \ref{ass2}, we have the following Lemma \ref{lem1} for the convergence of the loss functions $\widetilde{\mathcal{L}}_1$ and $\widetilde{\mathcal{L}}_2$.

\begin{lemma}[Theorems 4.6 \& 4.7 in \cite{yang2023pseudo}]\label{lem1}
    Under the Assumptions \ref{ass1}, \ref{ass2}. For an arbitrarily small $\varepsilon>0$, there exist two independent single-hidden-layer neural networks $\bm h$ and $\bm g$ such that 
\begin{equation}\label{conl_bound}
    0\le \widetilde{\mathcal{L}}_1 +\int_{\mathbb{R}^d} p_{\bm W}(\bm w) (\log p_{\bm W}(\bm w) - \log p_{{\bm Z}_1}(\bm z_1)) d\bm w <\varepsilon, \,\, \text{and} \quad \widetilde{\mathcal{L}}_2 < \varepsilon.
\end{equation}
\end{lemma}

\subsection{Monte Carlo integration error of loss function}
In practice, we are not able to define the loss function in the continuous form since the neural network is data-based optimization solver. Hence, we employ the Monte Carlo method (discrete form) to approximate the loss functions $\widetilde{\mathcal{L}}_1$ and $\widetilde{\mathcal{L}}_2$. To proceed, we first introduce the following two lemmas.

\begin{lemma}\label{lem2}
Let $\bm f\in L^1(\mathbb{R}^d,\mathbb{R})$ and dataset $\{\bm x^{(i)}\}_{i=1}^n$ follow the probability density function $p_{\bm X}(\bm x)$. Then we have
\begin{equation}
\mathbb{E}\left|\int_{\mathbb{R}^d}\bm f(\bm x)p_{\bm X}(\bm x)d\bm x-\frac{1}{n}\sum_{i=1}^n \bm f(\bm x^{(i)})\right|^2=\frac{1}{n} {\rm Var}(\bm f({\bm X})).
\end{equation}
\end{lemma}

\begin{proof}
We define a family of independent and identically distributed (i.i.d.) variables $\{{\bm X}^{(i)}\}_{i=1}^{n}$ that following distribution $p_{{\bm X}^{(i)}}(\bm x^{(i)})$. Then we have
\begin{equation} 
\begin{aligned}  &\mathbb{E}\left|\int_{\mathbb{R}^d}\bm f(\bm x)p_{\bm X}(\bm x)d\bm x-\frac{1}{n}\sum_{i=1}^n \bm f(\bm x^{(i)})\right|^2\\
=&\mathbb{E}\left|\sum_{i=1}^n\frac{\mathbb{E}[\bm f({\bm X}^{(i)})]-\bm f({\bm X}^{(i)})}{n}\right|^2 \\
=&\frac{1}{n^2}\left[\sum_{i=1}^n {\rm Var}(\bm f({\bm X}^{(i)}))+\sum_{i\neq j=1}^n {\rm Cov}(\bm f({\bm X}^{(i)}),\bm f({\bm X}^{(j)}))\right]=\frac{{\rm Var}(\bm f({\bm X}))}{n}.
\end{aligned}
\end{equation}
We complete the proof.
\end{proof}

\begin{lemma}\label{lem3}
Let $\bm f \in L^1(\mathbb{R}^d, \mathbb{R}$) be a Lipschitz continuous function with constant $K$, i.e. $|\bm f({\bm x})-\bm f({\bm y})|\le K \|{\bm x}-{\bm y}\|_2$. For a random variable ${\bm X} = ({X}_1, {X}_2, \ldots, {X}_d)\in \mathbb{R}^d$, we have the following inequality
\begin{equation}
    {\rm Var} (\bm f({\bm X})) \leq K^2 \sum_{i=1}^d {\rm Var}({X}_i).
\end{equation}
\end{lemma}

\begin{proof}
Let ${\bm Y}$ be an independent and identically distributed (i.i.d.) as ${\bm X}$. Since $\bm f(\bm X)$ and $\bm f(\bm Y)$ are independent and have the same variance, we have
\begin{equation}
\begin{aligned}
2 {\rm Var}(\bm f(\bm X)) &= {\rm Var}(\bm f(\bm X)) + {\rm Var}(\bm f(\bm Y)) \\
&= {\rm Var}(\bm f(\bm X) - \bm f(\bm Y)) = \mathbb{E}\left[|\bm f(\bm X) - \bm f(\bm Y)|^2\right] \\
&\leq K^2 \mathbb{E}\left[\|\bm X-\bm Y\|_2^2\right] \quad \text{by Lipschitz continuity} \\
&= K^2 \sum_{i=1}^d\mathbb{E}\left[|X_i-Y_i|^2\right] = K^2 \sum_{i=1}^d\mathbb{E}\left[(X_i)^2 - 2X_iY_i + (Y_i)^2\right] \\
&= K^2 \sum_{i=1}^d \left(2\mathbb{E}[(X_i)^2] - 2(\mathbb{E}[X_i])^2\right) = 2K^2 \sum_{i=1}^d {\rm Var}(X_i).
\end{aligned}
\end{equation}
We complete the proof.
\end{proof}

Based on the results of Lemmas \ref{lem2} and \ref{lem3}, we have the error estimate of the loss functions $\widetilde{\mathcal{L}}_1$ and $\widetilde{\mathcal{L}}_2$.

\begin{theorem}\label{them1}
Loss functions $\mathcal{L}_1$ in Eq.~\eqref{eq:loglike} and $\mathcal{L}_2$ in Eq.~\eqref{eq:loss2} are the discrete form of $\widetilde{\mathcal{L}}_1$ and $\widetilde{\mathcal{L}}_2$, respectively. For the forward problem ${\bm Y} = {\bm f}({\bm X}) + {\bm \varepsilon}({\bm X})$, $\bm X\in \mathbb{R}^d$ and $\bm Y \in \mathbb{R}^s$ with $N$ training samples,
we have
\begin{equation}\label{l1-l1}
 \max{\{\mathbb{E}|\widetilde{\mathcal{L}}_1-\mathcal{L}_1|^2, \mathbb{E}|\widetilde{\mathcal{L}}_2-\mathcal{L}_2|^2 \}}   \le\frac{C}{N}\left[\sum_{i=1}^{s}\left[\mathbb{E}[v_{ii}(\bm X)]+ {\rm Var}(f_i(\bm X))\right]+\sum_{i=1}^{d}{\rm Var}(X_i)\right],
\end{equation}
where constant $C$ depends on the Lipschitz constant $K$ in Lemma \ref{lem3}, and ${\bm v}({\bm x}) = {\rm Var}[{\bm \varepsilon}({\bm x})]$.
\end{theorem}

\begin{proof}
According to Lemma \ref{lem2} and Lemma \ref{lem3}, we have
\begin{equation}\label{I123}
    \begin{aligned}
        \mathbb{E}|\widetilde{\mathcal{L}}_1-\mathcal{L}_1|^2\le&\frac{C_1}{N}\sum_{i=1}^{d+s} {\rm Var}(W_i) =\frac{C_1}{N}\left[\sum_{i=1}^d {\rm Var}(X_i) + \sum_{i=1}^s {\rm Var}(Y_i)\right]\\
        =&\frac{C_1}{N}\left[\sum_{i=1}^d {\rm Var}(X_i) + \sum_{i=1}^s \mathbb{E}[Y_i^2]-\sum_{i=1}^s\mathbb{E}[Y_i]^2\right]:=\frac{C_1}{N}\left(I_1+I_2+I_3\right).
    \end{aligned}
\end{equation}

For the term $I_2$, we have
\begin{equation}\label{I2}
    \begin{aligned}        I_2=&\sum_{i=1}^s\int_{\mathcal{D}_{Y_i}}y_i^2P(Y_i=y_i)dy_i\\
=&\sum_{i=1}^s\int_{\mathcal{D}_{Y_i}}\int_{\mathcal{D}_{\bm X}}y_i^2P(Y_i=y_i|{\bm X}=\bm x)P({\bm X}=\bm x)d\bm x d y_i\\
=&\int_{\mathcal{D}_{\bm X}}P({\bm X}=\bm x)\int_{\mathcal{D}_{{Y_i}}}\sum_{i=1}^s y_i^2P(Y_i=y_i|{\bm X}=\bm x)d y_i d\bm x\\
=&\int_{\mathcal{D}_{\bm X}}P({\bm X}=\bm x)\sum_{i=1}^s[v_{ii}^2(\bm x)+ f_i^2(\bm x)]d\bm x  =\sum_{i=1}^s\mathbb{E}\left[v_{ii}(\bm X)+f_i^2(\bm X)\right].
    \end{aligned}
\end{equation}
where $\mathcal{D}_{\bm X}$ and $\mathcal{D}_{\bm Y}$ are domains of variables ${\bm X}$ and ${\bm Y}$. For the term $I_3$, we have
\begin{equation}\label{I3}
    \begin{aligned}        I_3=&\sum_{i=1}^s\left(\int_{\mathcal{D}_{Y_i}} y_iP(Y_i= y_i)dy_i\right)^2\\
=&\sum_{i=1}^s\left(\int_{\mathcal{D}_{Y_i}}\int_{\mathcal{D}_{\bm X}}y_iP(Y_i=y_i|\bm X=\bm x)P(\bm X=\bm x)d\bm x dy_i\right)^2\\
        =&\sum_{i=1}^s\left(\int_{\mathcal{D}_{\bm X}}P({\bm X}=\bm x)\int_{\mathcal{D}_{Y_i}}y_iP(Y_i=y_i|{\bm X}=\bm x)dy_i d\bm x\right)^2\\
        =&\sum_{i=1}^s\left(\int_{\mathcal{D}_{\bm X}}P({\bm X}=\bm x)f_i(\bm x)d\bm x\right)^2 = \sum_{i=1}^s\left(\mathbb{E}\left[ f_i({\bm X})\right]\right)^2.
    \end{aligned}
\end{equation}
Hence, the estimate for loss function $\mathcal{L}_1$ in Eq.~\eqref{l1-l1} is proved by combining inequalities  in Eqs.~\eqref{I123}, \eqref{I2} and \eqref{I3}. Similarly, we establish the estimates for  $\mathcal{L}_2$. We complete the proof.
\end{proof}

\subsection{Convergence of the KL divergence}
The section is to show that the convergence of $\widetilde{\mathcal{L}}_1$ and $\widetilde{\mathcal{L}}_2$ with certain assumptions implies the convergence of the KL-divergence between the conditional distribution $p_{\bm Y}$ and $p_{\widehat{\bm Y}}$.
\begin{assum}\label{ass3}
We assume that the target random variable $\bm W=(\bm X,\bm Y|\bm X)$ has the finite second-order moment, and there exists positive constant $A$ such that
\begin{equation}
    \nabla p_{\widehat {\bm W}}(\bm w)\le A(\|\bm w\|+1)p_{\widehat {\bm W}}(\bm w),
\end{equation}
where $p_{\widehat {\bm W}}(\bm w)$ is the distribution of the approximation ${\widehat {\bm W}}=(\bm X,{\widehat {\bm Y}}|\bm X)$.
\end{assum}

\begin{theorem}\label{them2}
Under the assumptions of Lemma \ref{lem1}, Theorem \ref{them1} and Assumption \ref{ass3}, for any given $\varepsilon > 0$, with sufficient training samples, there exist two neural networks $\bm h$ and $\bm g$ such that
    \begin{equation}
\mathbb{E}\left(\int_{\mathcal{D}_{\bm X}} D_{\rm KL}(p_{\bm Y|\bm X}\| p_{\widehat{\bm Y}|{\bm X}}) d\bm x \right)< \varepsilon.
    \end{equation}
\end{theorem}

\begin{proof}
In the training process, we assume the input variable $\bm X$ follows a uniform distribution, i.e., $P_X({\bm x}) = \frac{1}{|\mathcal{D}_{\bm X}|}$ for all ${\bm x}\in\mathbb{R}^d$. By the definition of KL divergence we have
\begin{equation}\label{Them2_1}
\begin{aligned}
&\int_{\mathcal{D}_{\bm X}} D_{\rm KL}(p_{\bm Y|\bm X}\| p_{\widehat{\bm Y}|{\bm X}}) d\bm x\\
=&|\mathcal{D}_{\bm X}|\int_{\mathcal{D}_{\bm X}} P_{\bm X}(\bm x) \left(\int_{\mathcal{D}_{\bm Y}} P_{\bm Y|\bm X}(\bm y|\bm x)\log\frac{P_{\bm Y|\bm X}(\bm y|\bm x)}{P_{\widehat{\bm Y}|{\bm X}}(\bm y|\bm x)} d\bm y\right) d\bm x\\
=&|\mathcal{D}_{\bm X}|\int_{\mathcal{D}_{\bm X}} \left(\int_{\mathcal{D}_{\bm Y}}P_{\bm X}(\bm x)  P_{\bm Y|\bm X}(\bm y|\bm x)\log\frac{P_{\bm Y|\bm X}(\bm y|\bm x) P_{\bm X}(\bm x) }{P_{\widehat{\bm Y}|{\bm X}}(\bm y|\bm x) P_{{\bm X}}(\bm x)} d\bm y\right) d\bm x\\
=&|\mathcal{D}_{\bm X}|\int_{\mathcal{D}_{\bm W}}P_{\bm W}(\bm w)\log\frac{P_{\bm W}(\bm w)}{P_{\widehat {\bm W}}(\bm w)} d\bm w=|\mathcal{D}_{\bm X}|\int_{\mathcal{D}_{\bm W}} D_{\rm KL}(p_{\bm W}\|p_{\widehat {\bm W}}) d\bm w.
    \end{aligned}
\end{equation}
Following the Theorem 4.9 in \cite{yang2023pseudo}, we have 
\begin{equation}\label{Them2_2}
\int_{\mathcal{D}_{\bm W}} D_{\rm KL}(p_{\bm W}\|p_{\widehat {\bm W}}) d\bm w < \widetilde{\mathcal{L}}_1 + \int_{\mathbb{R}^d} p_{\bm W}(\bm w) (\log p_{\bm W}(\bm w)- \log p_{{\bm Z}_1}(\bm z_1)) d\bm w + A(\widetilde{\mathcal{L}}_2 + \widetilde{\mathcal{L}}_2^{\frac{1}{2}}),
\end{equation}
where $A$ is the constant in Assumption \ref{ass3}.

We choose $\varepsilon_1$ and $\varepsilon_2$ such that
$\varepsilon_1<\varepsilon/(2|\mathcal{D}_{\bm X}|)$ and $\varepsilon_2<\min\{\varepsilon^2/(16A^2|\mathcal{D}_{\bm X}|^2), 1\}$. Let $$\mathcal{L}_1 < -\int_{\mathbb{R}^d} p_{\bm W}(\bm w) (\log p_{\bm W}(\bm w)- \log p_{{\bm Z}_1}(\bm z_1)) d\bm w + \varepsilon_1/2,$$ and $\mathcal{L}_2 < \varepsilon_2/2$. By Theorem \ref{them1} we can choose an integer $n$ large enough such that $$\max{\{\mathbb{E}|\widetilde{\mathcal{L}}_1-\mathcal{L}_1|^2, \mathbb{E}|\widetilde{\mathcal{L}}_2-\mathcal{L}_2|^2 \}} < \min \{\varepsilon_1^2/4, \varepsilon_2^2/4\},$$ then we have

\begin{equation}\label{Them2_3}
    \begin{aligned}
        &\mathbb{E} \left|\widetilde{\mathcal{L}}_1 + \int_{\mathbb{R}^d} p_{\bm W}(\bm w) (\log p_{\bm W}(\bm w)- \log p_{{\bm Z}_1}(\bm z_1)) d\bm w\right|\\
        <& \mathbb{E} \left|\mathcal{L}_1 + \int_{\mathbb{R}^d} p_{\bm W}(\bm w) (\log p_{\bm W}(\bm w)- \log p_{{\bm Z}_1}(\bm z_1)) d\bm w\right| + \mathbb{E}|\widetilde{\mathcal{L}}_1-\mathcal{L}_1|\\
        <&\frac{\varepsilon_1}{2} + (\mathbb{E}|\widetilde{\mathcal{L}}_1-\mathcal{L}_1|^2)^\frac{1}{2} = \frac{\varepsilon_1}{2} + \sqrt{\frac{\varepsilon_1^2}{4}} = \varepsilon_1,
    \end{aligned}
\end{equation}
and 
\begin{equation}\label{Them2_4}
    \mathbb{E}(A(\widetilde{\mathcal{L}}_2 + \widetilde{\mathcal{L}}_2^{\frac{1}{2}})) 
    = A(\mathbb{E}(\widetilde{\mathcal{L}}_2) + \mathbb{E}(\widetilde{\mathcal{L}}_2^{\frac{1}{2}})) 
    \le A(\mathbb{E}(\widetilde{\mathcal{L}}_2) + (\mathbb{E}(\widetilde{\mathcal{L}}_2)^{\frac{1}{2}})
    < A(\varepsilon_2 + \varepsilon_2^{\frac{1}{2}}).
\end{equation}
The last inequality holds due to 
\begin{equation}\label{Them2_5}
    \mathbb{E}(\widetilde{\mathcal{L}}_2) 
    \le \mathbb{E}(\mathcal{L}_2) + \mathbb{E}|\widetilde{\mathcal{L}}_2-\mathcal{L}_2|
    \le \mathbb{E}(\mathcal{L}_2) + (\mathbb{E}|\widetilde{\mathcal{L}}_2-\mathcal{L}_2|^2)^\frac{1}{2}
    <\frac{\varepsilon_2}{2} + \frac{\varepsilon_2}{2} = \varepsilon_2.
\end{equation}
Thus, by \eqref{Them2_1}-\eqref{Them2_4} we have
\begin{equation}\label{Them2_6}
\begin{aligned}
\mathbb{E}\left(\int_{\mathcal{D}_{\bm X}} D_{\rm KL}(p_{\bm Y|\bm X}\| p_{\widehat{\bm Y}|{\bm X}}) d\bm x\right) = &|\mathcal{D}_{\bm X}|\mathbb{E}\left(\int_{\mathcal{D}_{\bm W}} D_{\rm KL}(p_{\bm W}\|p_{\widehat {\bm W}}) d\bm w\right)\\
<&|\mathcal{D}_{\bm X}|\left(\varepsilon_1 + A(\varepsilon_2 + \varepsilon_2^{\frac{1}{2}})\right) \le |\mathcal{D}_{\bm X}|\left(\varepsilon_1 + A(2\varepsilon_2^{\frac{1}{2}})\right)\\
<&|\mathcal{D}_{\bm X}|\left( \frac{\varepsilon}{2|\mathcal{D}_{\bm X}|} + 2A\left(\frac{\varepsilon^2}{16A^2|\mathcal{D}_{\bm X}|^2}\right)^\frac{1}{2}\right) = \varepsilon.
\end{aligned}
\end{equation}
We complete the proof.
\end{proof}

\section{Numerical examples}\label{sec:example}
In this section we present numerical experiments to demonstrate the performance of the proposed conditional PR-NF model. The synthetic example in Sec.~\ref{sec:ex1} benchmarks the accuracy of our model for problems with  known analytical ground-truth solutions. 
In Sec.~\ref{ex_high} we present the implementation of the PR-NF model on high-dimensional uncertainty problems. Sec.~\ref{gcs} presents an application to a real-world geologic carbon storage
problem. In all numerical simulations, the PyTorch machine learning framework has been implement with CUDA GPU. The source code is publicly available at \url{https://github.com/mlmathphy/PRNF_uncertainty}. The numerical results presented in this section can be accurately replicated by utilizing the code available on our GitHub repository.

\subsection{Verification of algorithm accuracy}\label{sec:ex1}
We consider the following uncertainty propagation problem
\begin{equation}
    { y} = f({ x}) + \varepsilon({ x}),  \,{ x} \in \mathcal{D} = [0,1],
\end{equation}
where the deterministic function $f({ x})$  is considered  as two expressions:
\begin{equation}\label{eq_ex1_f}
     {\bf Quadratic:}\, f({ x}) = 4({ x}-0.5)^2 \text{ or }\, {\bf Sin:} \, f({ x}) = \sin{(2\pi { x})}.
\end{equation}
The uncertainty $\varepsilon({ x})$ has four options (homoscedastic and heteroscedastic):
\begin{equation}\label{eq_ex1_e}
\begin{aligned}
& {\bf Gaussian: } \, \varepsilon({ x}) \sim \mathcal{N}(0,0.15) \text{ or }\, \varepsilon({ x}) \sim \mathcal{N}(0,0.2 |f({ x})|).\\
& {\bf Laplace: } \, \varepsilon({ x}) \sim {\rm Laplace}(0,0.1) \text{ or }\, \varepsilon({ x}) \sim {\rm Laplace}(0,0.15|f({ x})|).
\end{aligned}
\end{equation}
The training dataset is based on 20K samples of  variable $\{{x}^{(i)}\}_{i=1}^{N_{\rm train}}$ uniformly generated in $\mathcal{D}$. For each ${x}^{(i)}$, $i = 1,\ldots, N_{\rm train}$, we generated the corresponding observation ${y}^{(i)}$ using the equation ${y}^{(i)} = f({x}^{(i)}) + \varepsilon({x}^{(i)})$. The pair $\{{x}^{(i)},{y}^{(i)}\}$, $i = 1,\ldots, N_{\rm train}$, consists of the training dataset. 

In the context of the forward problem, ${x}$ serves as the input, and ${y}$ represents the output. Conversely, in the inverse problem, the roles of ${x}$ and ${y}$ are reversed. We utilize the same training dataset $\{x^{(i)},y^{(i)}\}_{i=1}^{N_{\rm train}}$ and build two independent single hidden layer neural networks for forward and inverse problems, respectively.
In the training process, we choose hyperparameter $\lambda$, which controls the relative importance of the pseudo-reversivility and the
negative log-likelihood losses, such that the cross entropy $H(\lambda)$ in Eq.~\eqref{eq:cross} has a minimum. In both forward and inverse problems we choose $\lambda = 80$. The single hidden layer neural network has 256 neurons and uses Tanh activation function. The neural networks are trained by the Adam optimizer with 2000 epochs.

\subsubsection{The PR-NF model is capable to sample and evaluate the conditional distribution $p({ y} | { x})$}
We use the forward uncertainty propagation (${x}$ serves as the input, and ${y}$ represents the output) as an example.
This section is to demonstrate that the well-trained PR-NF model is capable to be served as a surrogate model to 
\begin{itemize}
    \item generate sufficient samples of observation variable $y$ for any input variable $x \in \mathcal{D}$;
    \item evaluate the conditional distribution $p({ y} | { x})$ based on samples of $y$.
\end{itemize}
The Kullback–Leibler (KL) divergence is applied as a metric to test the accuracy of models.

Fig.~\ref{fig_test1_kl} shows the Kullback–Leibler  divergence between the ground-truth $p({y} | {x})$ and the approximation $p(\widehat{ y} | {x})$ at $x\in (-1,2)$ formulated as 
\begin{equation}\label{eqkl_1}
 D_{\rm KL}(x)= \int_{\mathbb{R}^{s}} p({y}|{x})\log{\left(\frac{p({y}|{x})}{{p}(\widehat{y}|{x})}\right)}d {y},
\end{equation}
where the density $p(\widehat{ y} | {x})$ is calculated based on 20K samples $\{\widehat{y}^{(i)}\}_{i=1}^{N_{\rm sample}}$  generated from the well-trained PR-NF model and the KL-divergence is numerically approximated by Riemann sum over a uniform mesh. 
It is shown that the PR-NF model achieves a good agreement in KL-divergence when the input $x$ is in the training dataset domain, i.e., $x\in \mathcal{D}$. However, the KL-divergence gets worse when $x\not\in \mathcal{D}$, which means the PR-NF model does
not have the prediction property for adapting beyond its initial training domain.

\begin{figure}[h!]
    \centering
  {\includegraphics[width=0.95\textwidth]{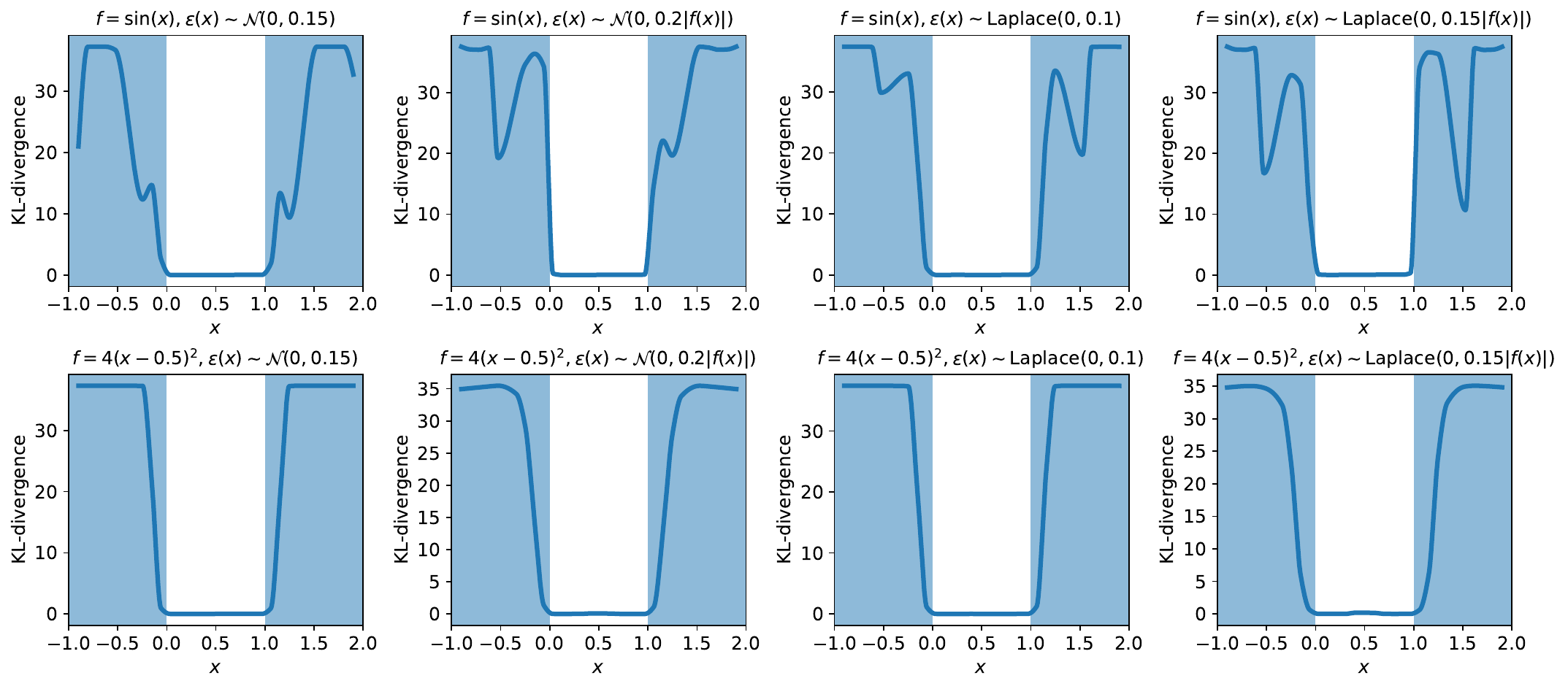}}
  \vspace{-0.2cm}
    \caption{The Kullback–Leibler (KL) divergence between the ground-truth density $p({ y} | { x})$ and the approximation $p(\widehat{ y} | {x})$ from the PR-NF model for any $x\in (-1,2)$. The top row is for the case $f({ x}) = \sin{(2\pi { x})}$ with four different additive noises and the bottom row corresponds to $f({ x}) = 4({ x}-0.5)^2$. The figure shows that the PR-NF model performs well in terms of KL-divergence within its training range, specifically for $x\in (0,1)$. It is noted that, when input $x$ falls outside this domain, the KL-divergence does increase, which means the PR-NF model does not have the prediction property for  adapting beyond its training domain.}
    \label{fig_test1_kl}
\end{figure}

Building upon the conclusion from Fig.~\ref{fig_test1_kl}, we investigate further into the performance of the PR-NF model regarding the evaluation of the density function $p({y} | {x})$ at specific values of $x$. Figs.~\ref{fig_test1_sin} and \ref{fig_test1_square} show the case of $f({ x}) = \sin{(2\pi { x})}$ and $f({ x}) = 4({ x}-0.5)^2$, respectively. Each row represents different additive noise and each column corresponds to different test point $x$. In alignment with the findings depicted in Fig.~\ref{fig_test1_kl}, we select four test points, $x = -0.8, 0.2, 0.8, 1.8$, with two points falling outside the domain $\mathcal{D}$ and two points inside. In each plot, the orange curve is the ground-truth $p({y} | {x})$ and the blue histogram is the approximation $p(\widehat{ y} | {x})$ generated by the PR-NF model with $N_{\rm sample} = 20$K evaluation samples. Very good agreements are observed for inside test point $x\in \mathcal{D}$ (2nd \& 3rd columns) and the PR-NF model captures different types of uncertainties including the heteroscedastic noises, but it does not work for test points outside of domain $\mathcal{D}$ (1st \& 4th columns). The center of distribution function $p({y} | {x})$ locates at $f(x)$ and the shape of distribution is determined by the noise $\varepsilon(x)$. Here we highlight that the well-trained PR-NF model can evaluate $p(y|x)$ at any $x\in \mathcal{D}$ without any additional training processes.

\begin{figure}[h!]
    \centering
  {\includegraphics[width=0.95\textwidth]{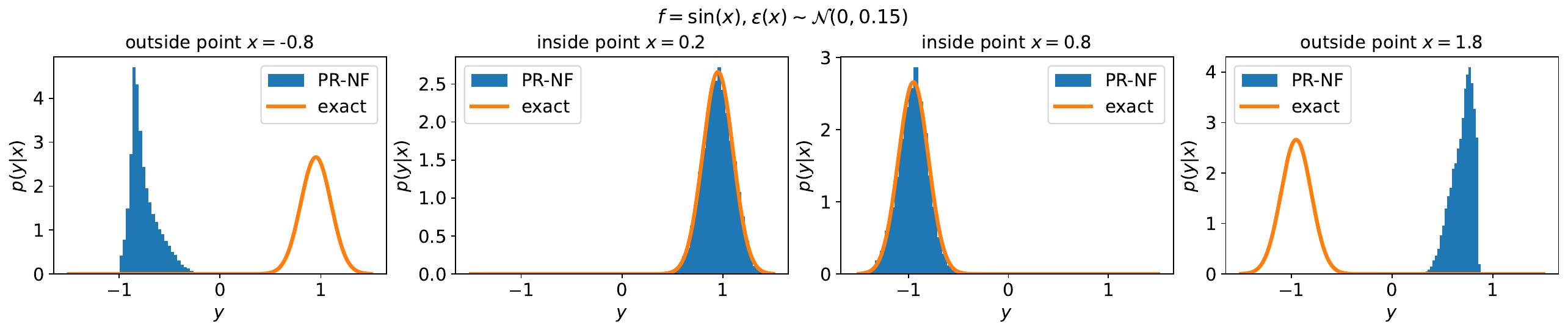}}
  {\includegraphics[width=0.95\textwidth]{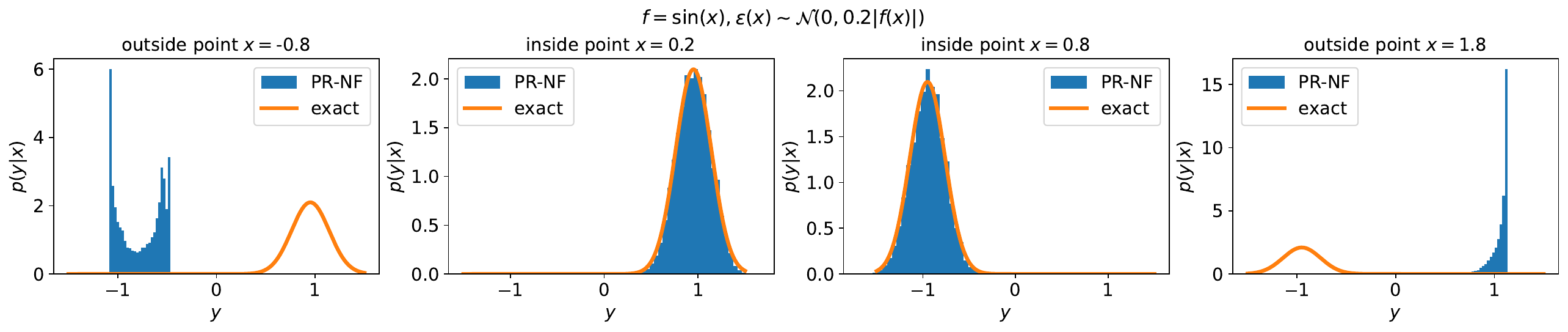}}
  {\includegraphics[width=0.95\textwidth]{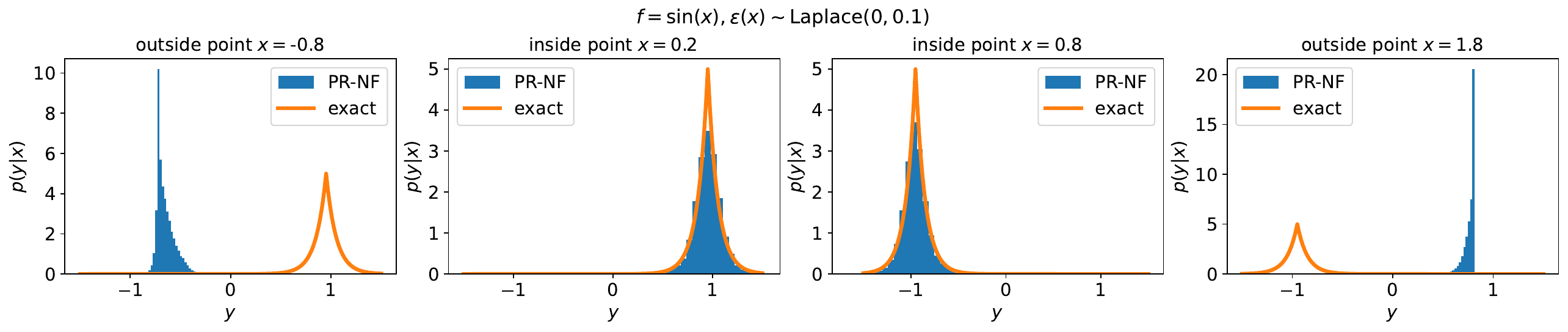}}
{\includegraphics[width=0.95\textwidth]{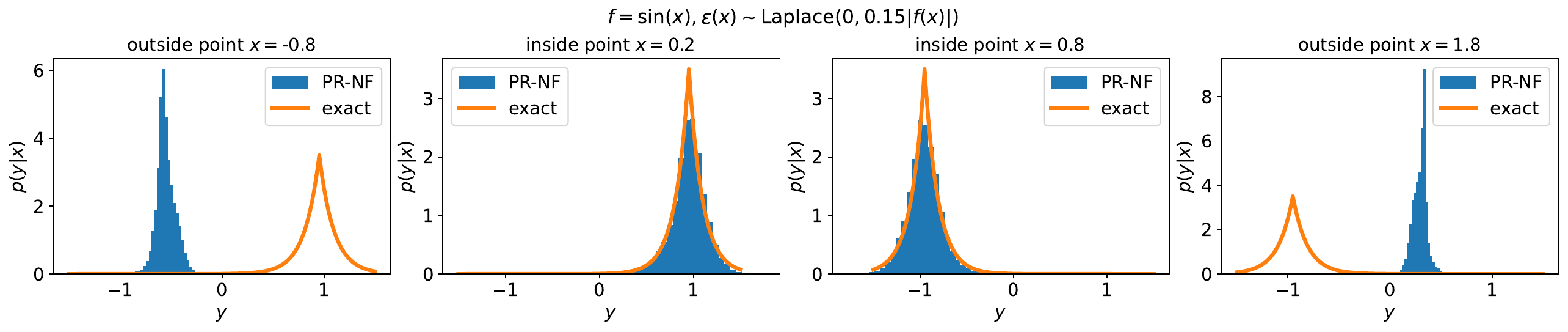}}
\vspace{-0.25cm}
\caption{The accuracy performance of the well-trained PR-NF model on evaluating $f(x) = \sin{2\pi x}$ with four different additive noises. Each row represents different noise and each column corresponds to different point $x = -0.8, 0.2, 0.8, 1.8$ (left to right). Consistent with the conclusion from Fig.~\ref{fig_test1_kl}, very good agreements are observed for
inside point $x \in D$ (2nd \& 3rd columns). However, it does not work for  points outside of domain $D$ (1st \& 4th columns), which means  the PR-NF model does not have the prediction property for adapting beyond its initial training domain.}
    \label{fig_test1_sin}
\end{figure}

\begin{figure}[h!]
    \centering
  {\includegraphics[width=0.95\textwidth]{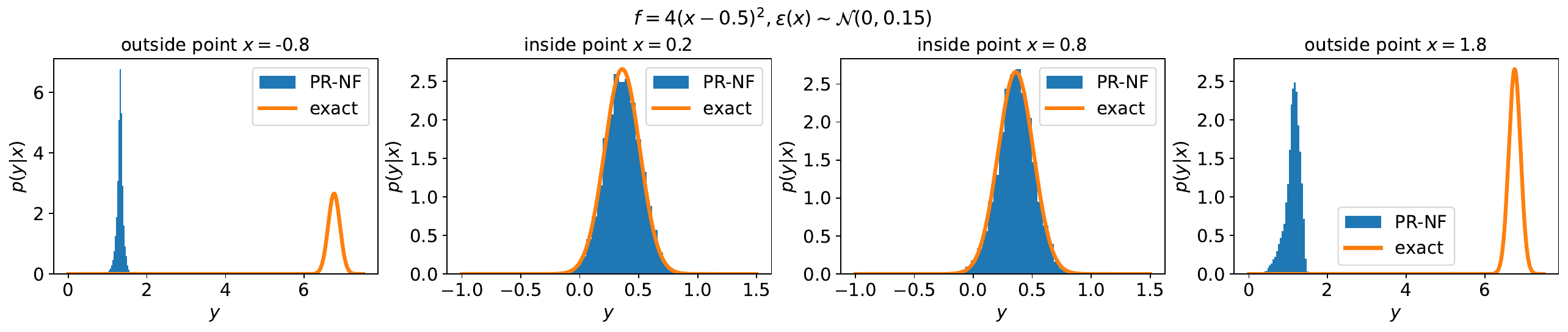}}
  {\includegraphics[width=0.95\textwidth]{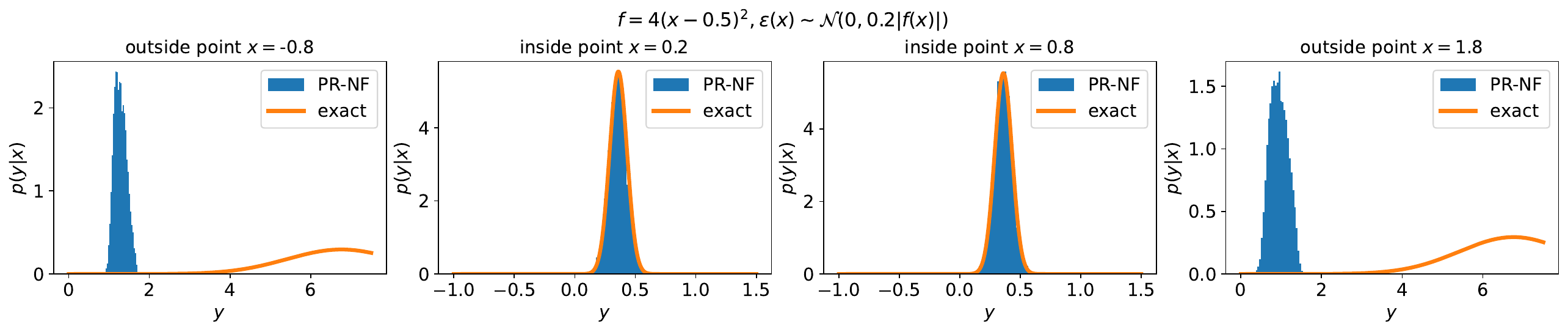}}
  {\includegraphics[width=0.95\textwidth]{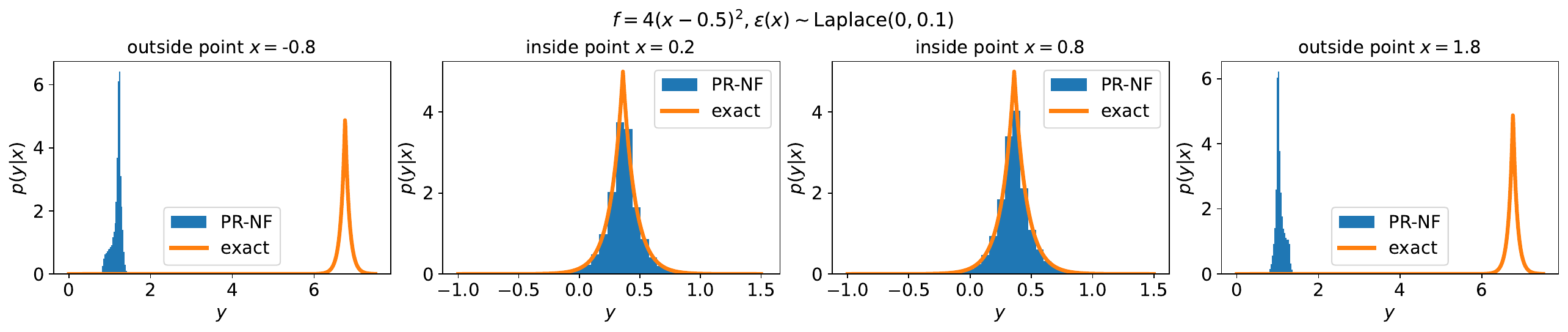}}
    {\includegraphics[width=0.95\textwidth]{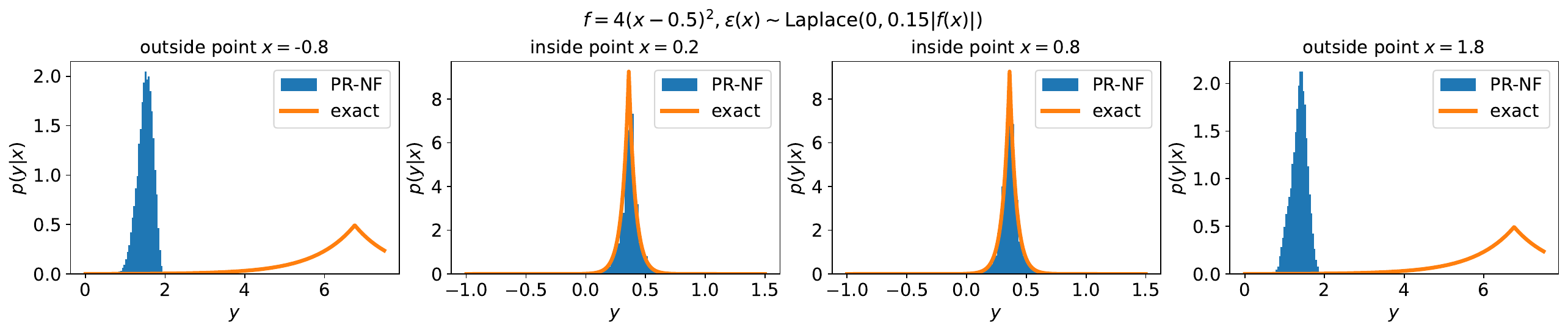}}
    \caption{
 The accuracy performance of the well-trained PR-NF model on evaluating $f(x) = 4(x-0.5)^2$ with four different additive noises. Each row represents different noise and each column corresponds to different point $x = -0.8, 0.2, 0.8, 1.8$ (left to right). Consistent with the conclusion from Fig.~\ref{fig_test1_kl}, very good agreements are observed for
inside point $x \in D$ (2nd \& 3rd columns).
    }
    \label{fig_test1_square}
\end{figure}

\subsubsection{The PR-NF model can evaluate the bimodal conditional distribution $p({x} | {y})$}
Fig.~\ref{fig_test1_inv} illustrates the training data of two physical models: $f(x) = \sin{2\pi x}$ (left panel) and $f(x) = 4(x-0.5)^2$ (right panel), each subject to uncertainty $\varepsilon \sim \mathcal{N}(0,0.15)$. It is observed that for each observation $y$, the corresponding parameter variable $x$ exhibits a probability distribution characterized by two distinct peaks (bimodal distribution). Due to the lack of a one-to-one relationship from variable $y$ to $x$, we are not able to derive a function of $x$ in terms of $y$. Traditional neural network models often struggle to effectively capture the process from $y$ to $x$ since mean squared error(MSE) fails to capture two peaks.

\begin{figure}[h!]
    \centering
  {\includegraphics[scale=0.25]{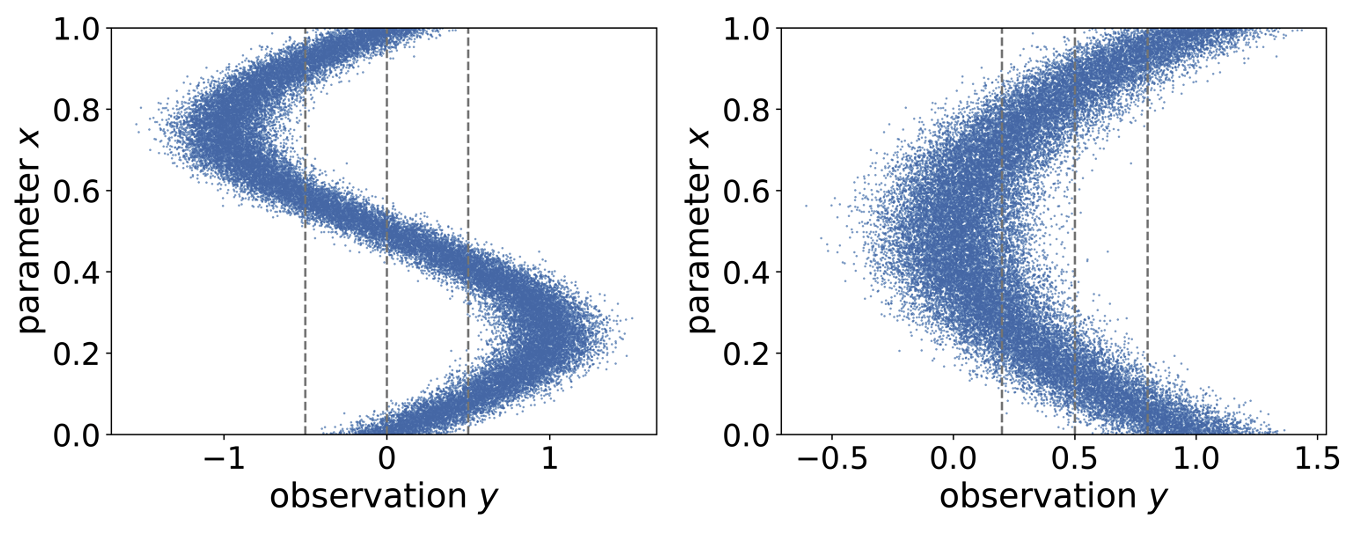}}
  \vspace{-0.2cm}
    \caption{The training data $\{x^{(i)},y^{(i)}\}_{i=1}^{N_{\rm train}}$ of  functions $f(x) = \sin{(2\pi x)}$ (left panel) and $f(x) = 4(x-0.5)^2$ (right panel), each subject to $\varepsilon \sim \mathcal{N}(0,0.15)$. Dashed lines are three test points used in Figs.~\ref{inv_sin_gaussian_homo} and \ref{inv_square_gaussian_homo}. }
    \label{fig_test1_inv}
    \vspace{-0.2cm}
\end{figure}

Fig.~\ref{inv_sin_gaussian_homo} shows the conditional distribution $p(x|y)$ constructed by the well-trained PR-NF model at three test points $y = -0.5,0,0.5$, where $f(x) = \sin{2\pi x}$ and $\varepsilon \sim \mathcal{N}(0,0.15)$. Even though given the explicit expression of function $f$ and uncertainty $\varepsilon$, we are not able to derive the analytical expression of $p(x|y)$. The ground truth PDF is constructed using Gaussian kernel density estimation with sufficient samples. Very good agreements on the distribution are observed. 
Similarly, Fig.~\ref{inv_square_gaussian_homo} shows the conditional distribution $p(x|y)$ at test points $y = 0.25,0.5,0.75$, where $f(x) = 4(x-0.5)^2$ and $\varepsilon \sim \mathcal{N}(0,0.15)$.
It is observed that the PR-NF model is not sensitive to the input-output function relationship and has the capacity to capture the bimodal distribution. It utilize the identical training data and similar neural network configuration (with a simple input and output position switch) for both forward and inverse uncertainty problems.

\begin{figure}[h!]
    \centering
  {\includegraphics[scale=0.4]{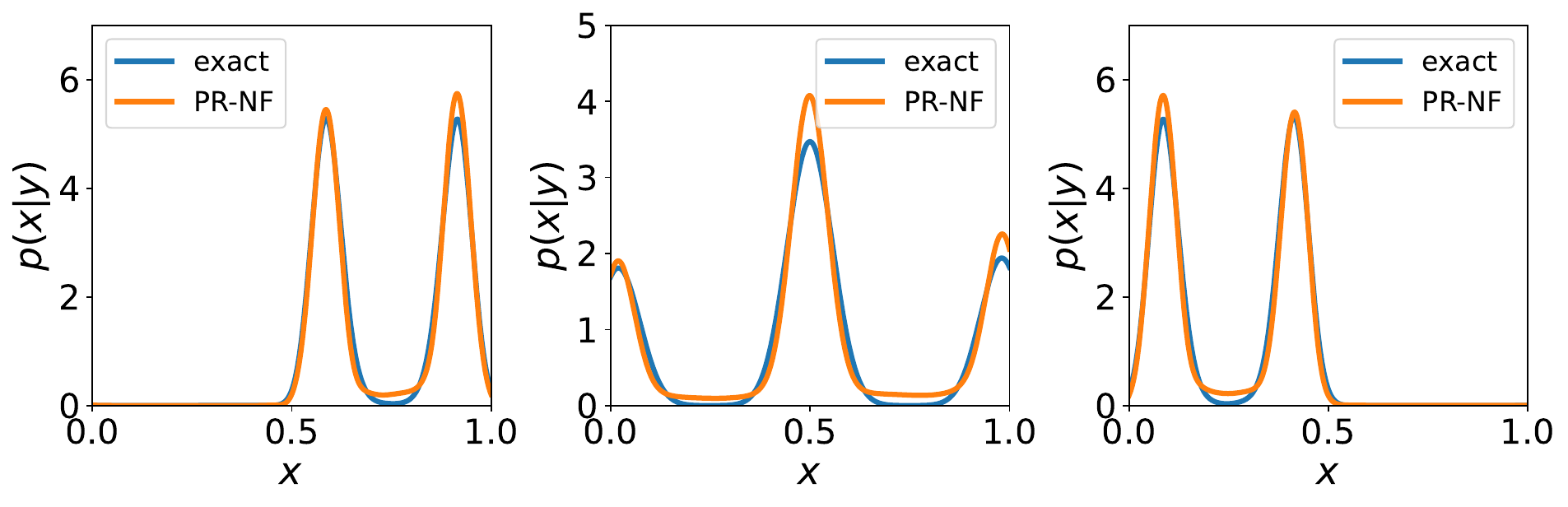}}
  \vspace{-0.3cm}
    \caption{The PR-NF model evaluates the conditional distribution $p(x|y)$ at three  points $y = -0.5,0,0.5$, where $f(x) = \sin{(2\pi x)}$ and $\varepsilon \sim \mathcal{N}(0,0.15)$. The well-trained PR-NF model can capture the probability distribution characterized by
two distinct peaks (bimodal distribution) and quantify the inverse uncertainty propagation well.}
    \label{inv_sin_gaussian_homo}
\end{figure}

\begin{figure}[h!]
    \centering
  {\includegraphics[scale=0.4]{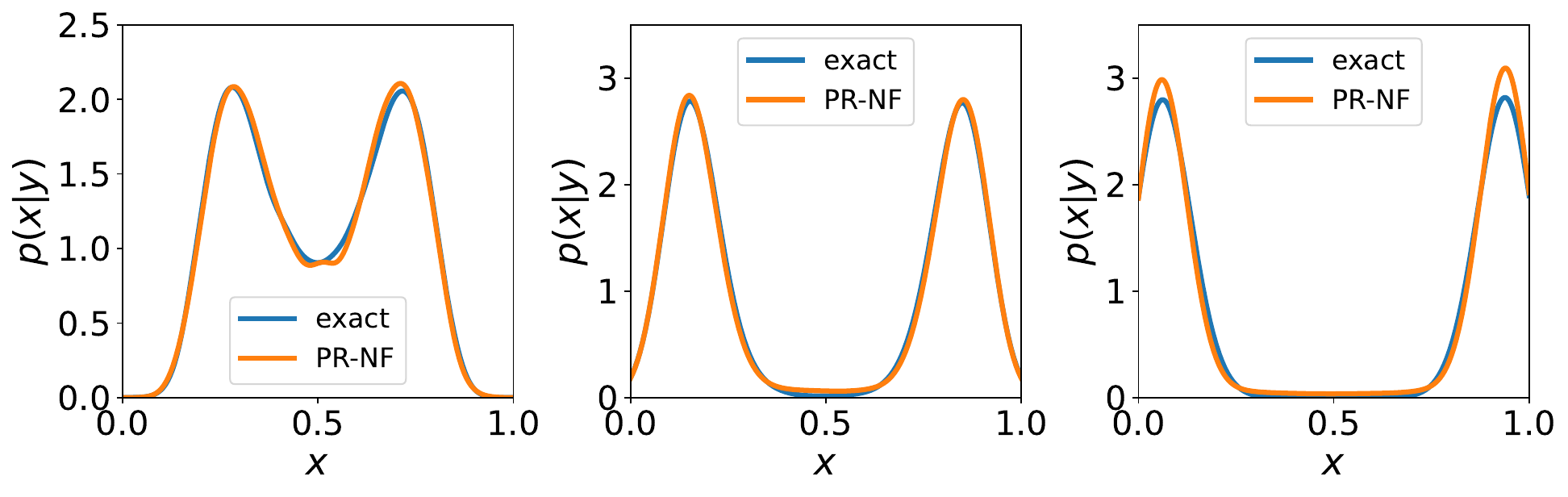}}
  \vspace{-0.3cm}
    \caption{
    The PR-NF model evaluates the conditional distribution $p(x|y)$ at three  points $y = 0.25,0.5,0.75$, where $f(x) = 4(x-0.5)^2$ and $\varepsilon \sim \mathcal{N}(0,0.15)$. The well-trained PR-NF model can capture the probability distribution characterized by
two distinct peaks (bimodal distribution) and quantify the inverse uncertainty propagation well.}
    \label{inv_square_gaussian_homo}
\end{figure}

\subsection{High-dimensional regression problem}\label{ex_high}
This numerical example aims to illustrate how the PR-NF model effectively quantifies uncertainty in high-dimensional problems. We consider a physical model that has the parameter variable ${\bm x}\in \mathcal{D}=[0,1]^{d}$ with $d = 20$ and the observation variable ${\bm y}\in \mathbb{R}^{s}$ with $s=5, 10, 20$, as follows
\begin{equation}
{\bm y} = {\bm f}({\bm x}) + {\bm \varepsilon}({\bm x}),
\end{equation}
where the deterministic function $ {\bm f}({\bm x}) = {\bf A} {\bm x}$ and values of the coefficient matrix ${\bf A}\in \mathbb{R}^{s\times d}$ are selected from the uniform distribution $\mathcal{U}[0,1]$. The additive noise ${\bm \varepsilon}$ has three options:
\begin{itemize}
    \item A multivariate normal distribution without correlation i.e., ${\bm \varepsilon}({\bm x}) \sim \mathcal{N}(0, 0.1\cdot \mathbf{I}_{s})$.
    \item A multivariate Gaussian mixture distribution that consists of two same weight Gaussian distribution components, ${\bm \varepsilon}_1({\bm x}) \sim \mathcal{N}(0.1, 0.1\cdot \mathbf{I}_{s})$ and ${\bm \varepsilon}_2({\bm x}) \sim \mathcal{N}(-0.1, 0.1\cdot \mathbf{I}_{s})$.
    \item A multivariate normal distribution with correlation, i.e., ${\bm \varepsilon}({\bm x}) \sim \mathcal{N}(0, \Sigma_{s\times s})$, where the covariance matrix $\Sigma_{s \times s}$ is a symmetric positive definite matrix and its values are randomly selected from $\mathcal{N}(0,1)$.
\end{itemize}
In this problem the training dataset $\mathcal{V}$ consists of $N_{\rm train}=30$K samples. 
The hyperparameter $\lambda$ is chosen as $\lambda = 80$ for which the cross entropy, $H(\lambda)$, has a minimum.
We use a fully-connected neural network with one hidden layer to build a surrogate model for the conditional distribution $P(\bm y | \bm x)$.
We evaluate the PR-NF model by the following metrics:
\vspace{0.2cm}
\begin{itemize}[leftmargin=15pt]
    \item Error of the normalized mean value:
    \begin{equation}
    \label{eq_avgmean}
        {\rm Err_{mean}} = \frac{1}{N_{\rm test}}\sum_{i=1}^{N_{\rm test}} \frac{\|{\bm f}({\bm x}^{(i)})-\widehat{\bm f}({\bm x}^{(i)})\|_2}{\|{\bm f}({\bm x}^{(i)})\|_2},
    \end{equation}
    where ${\bm f}({\bm x}^{(i)})$ is computed as the mean of ${\bm y}({\bm x}^{(i)})$. 
    \item Error of the standard deviation, and covariance matrix that for the third additive noise (the correlated normal distribution):
    {\footnotesize	
    \begin{equation} 
        {\rm Err_{std}} = \frac{1}{N_{\rm test}}\sum_{i=1}^{N_{\rm test}} \frac{\|{\bm \sigma}({\bm y}({\bm x}^{(i)}))-{\bm \sigma}(\widehat{\bm y}({\bm x}^{(i)}))\|_2}{\sqrt{s}},\quad         {\rm Err_{cov}} = \frac{1}{N_{\rm test}}\sum_{i=1}^{N_{\rm test}} \frac{\|{\Sigma}({\bm y}({\bm x}^{(i)}))-{\bm \Sigma}(\widehat{\bm y}({\bm x}^{(i)}))\|_F}{s}.
    \end{equation}
    }
    \item Average of the KL divergence:
    \begin{equation}
    \label{eq_avgKL}
         {\rm Avg_{KL}} = \frac{1}{N_{\rm test}}\sum_{i=1}^{N_{\rm test}} D_{\rm KL}(p({\bm x}^{(i)}) \,\| \,\widehat{p}({\bm x}^{(i)})),
    \end{equation}
    where the KL divergence is defined as the relative entropy from the approximate density (generated from PR-NF) $p_{\rm approx}$ to the exact density $p_{\rm exact}$, i.e.,
\begin{equation}\label{eq_kl}
 D_{\rm KL}(p({\bm x}^{(i)}) \,\| \,\widehat{p}({\bm x}^{(i)}))= \int_{\mathbb{R}^{s}} p({\bm y}|{\bm x}^{(i)})\log{\left(\frac{p({\bm y}|{\bm x}^{(i)})}{{p}(\widehat{\bm y}|{\bm x}^{(i)})}\right)}d {\bm y},
\end{equation}
where $D_{\rm KL}$ is approximated by the Riemann sum over a uniform mesh.
\end{itemize}
\vspace{0.2cm}
These average metrics are averaged among $N_{\rm test}=100$ test points, which are randomly sampled in domain $\mathcal{D}$. For each test point ${\bm x}^{(i)}$, we generate $N_{\rm sample}=20$K samples of ${\bm y}({\bm x}^{(i)})$ for the evaluation. $\|\cdot \|_2$ and $\|\cdot \|_F$ are $L^2$ and Frobenius norm among dimensionality, respectively. 

To assess the robustness of the PR-NF model, we conducted ten (without repeats) random simulations, each varying the hyperparameter $\lambda$ and the number of neurons $N_{\rm neuron}$  from the set $$\{(\lambda,N_{\rm neuron}) | \lambda \in \{1, 50, 100, 200\}, N_{\rm neuron} \in \{400, 600, 800, 1000\}\}.$$
As the PR-NF employs a single hidden layer neural network, the number of hidden layers is not a hyperparameter and remains fixed at one. Figures \ref{fig_dim5}, \ref{fig_dim10}, and \ref{fig_dim20} depict the decay of training loss alongside the above defined average metrics of testing dataset, each presented with a 95\% confidence interval. Each figure represents a different dimensional case ($s=5, 10, 20$), with each row corresponding to different types of additive noise.
The decay trends observed in the figures exhibit strong performance, particularly noteworthy is stable in error decay following a short transition period, evident by the narrow confidence intervals. These numerical findings underscore the capability of the proposed PR-NF model in addressing various high-dimensional uncertainty challenges. The PR-NF model consistently converges towards the minimum loss within a reasonable parameter range, demonstrating its robustness and stability. Regarding computational efficiency, with the GPU-acceleration, even for the case $s=20$ the training process takes about 10 minutes and the evaluation is completed in seconds.

\begin{figure}[h!]
    \centering
  {\includegraphics[width=0.95\textwidth]{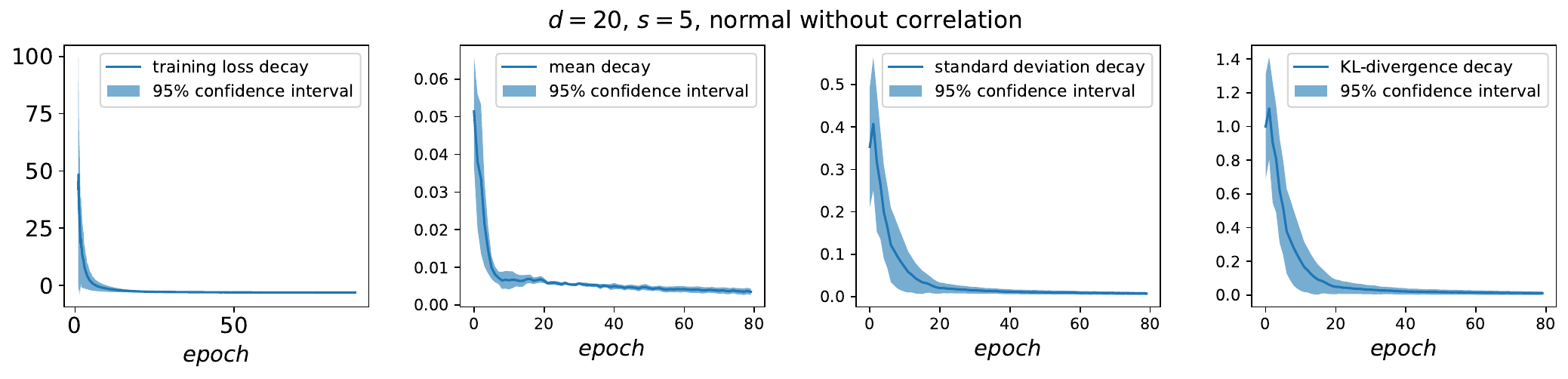}}
  {\includegraphics[width=0.95\textwidth]{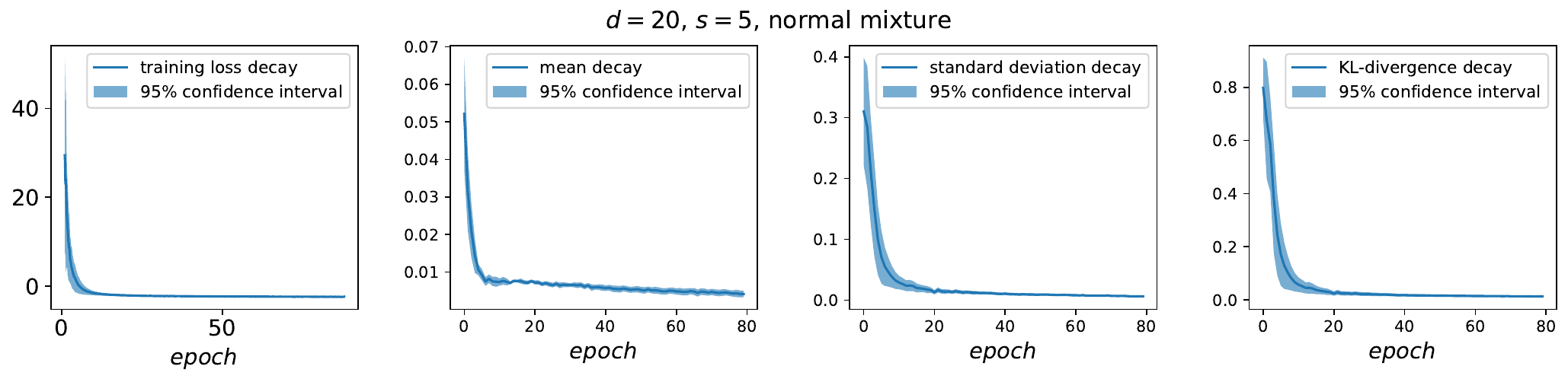}}
    {\includegraphics[width=0.95\textwidth]{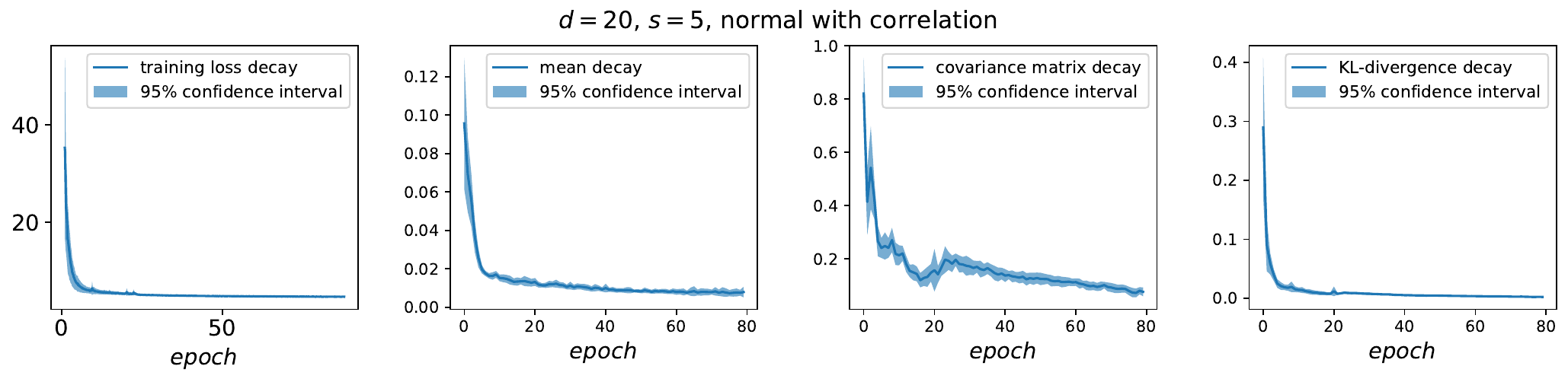}}
    \caption{The decay of training loss and average metrics defined by Eqs.~\eqref{eq_avgmean}-\eqref{eq_avgKL} are displayed alongside 95\% confidence intervals, based on ten simulations of randomly selected hyperparameters. The dimensionality of physical model is $d=20$, $s=5$, with three types of additive noise: uncorrelated multivariate normal distribution (top row), Gaussian mixture distribution with equal weight components (middle row), and correlated multivariate normal distribution (bottom row). The errors consistently converge to the minimum loss within a reasonable parameter range, showing stable trends after a brief transition period. These findings underscore the PR-NF model's effectiveness in addressing high-dimensional uncertainty challenges, highlighting its robustness and stability.}
    \label{fig_dim5}
\end{figure}

\begin{figure}[h!]
    \centering
  {\includegraphics[width=0.95\textwidth]{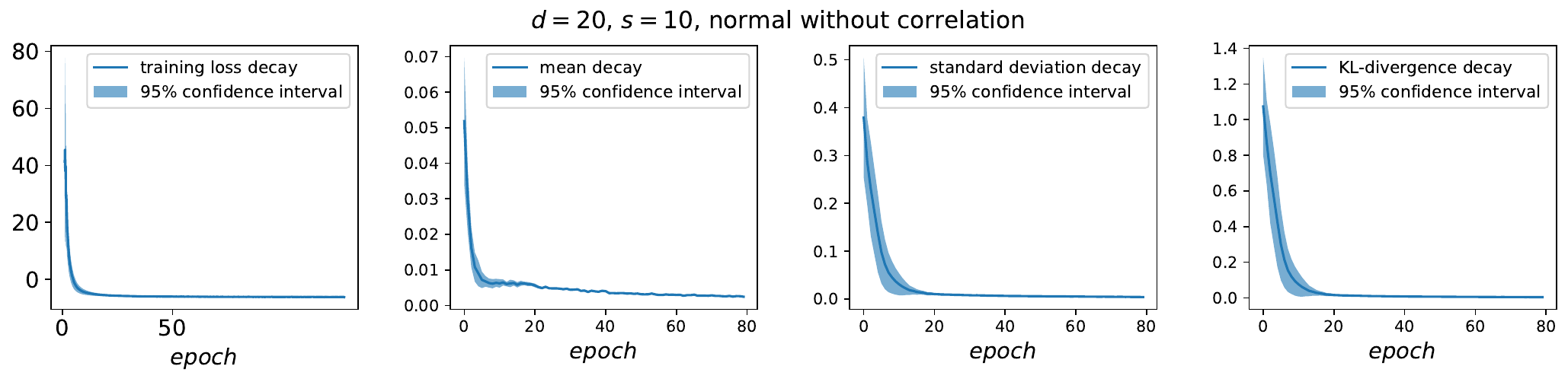}}
  {\includegraphics[width=0.95\textwidth]{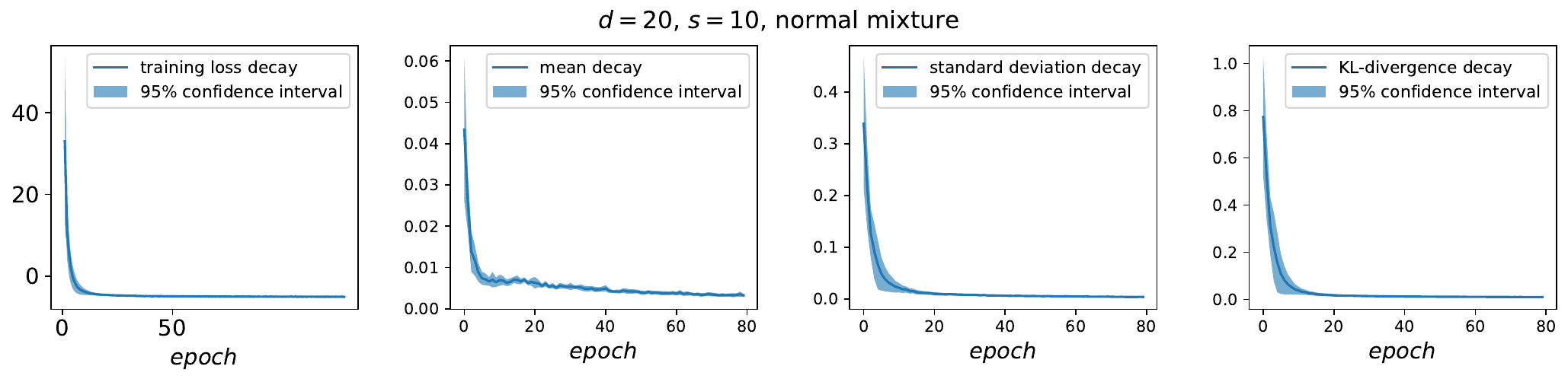}}
  {\includegraphics[width=0.95\textwidth]{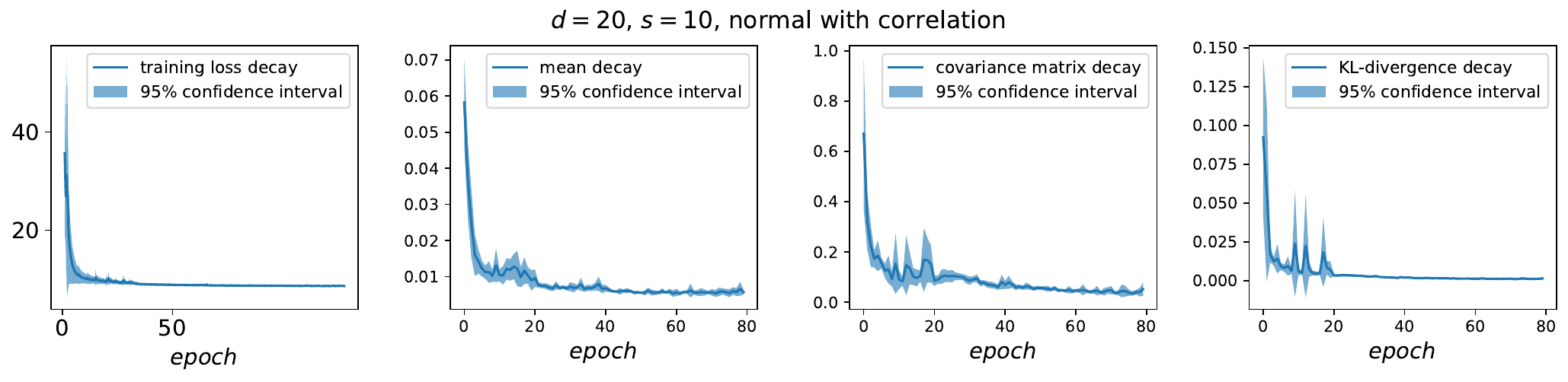}}
    \caption{
    The decay of training loss and average metrics defined by Eqs.~\eqref{eq_avgmean}-\eqref{eq_avgKL} are displayed alongside 95\% confidence intervals, based on ten simulations of randomly selected hyperparameters. The dimensionality of physical model is $d=20$, $s=10$, with three types of additive noise: uncorrelated multivariate normal distribution (top row), Gaussian mixture distribution with equal weight components (middle row), and correlated multivariate normal distribution (bottom row). The errors consistently converge to the minimum loss within a reasonable parameter range, showing stable trends after a brief transition period. These findings underscore the PR-NF model's effectiveness in addressing high-dimensional uncertainty challenges, highlighting its robustness and stability.}
    \label{fig_dim10}
\end{figure}

\begin{figure}[h!]
    \centering
  {\includegraphics[width=0.95\textwidth]{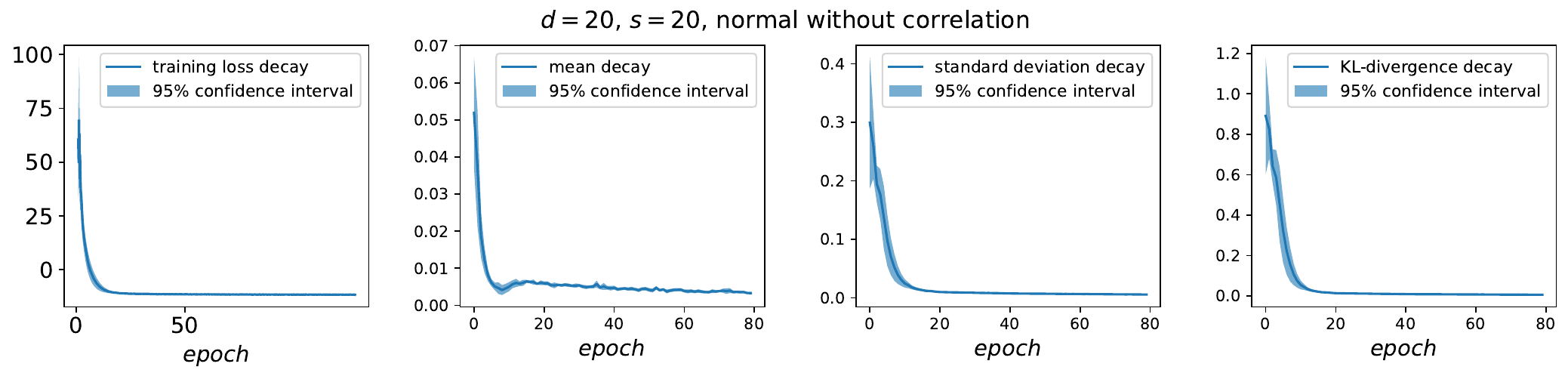}}
  {\includegraphics[width=0.95\textwidth]{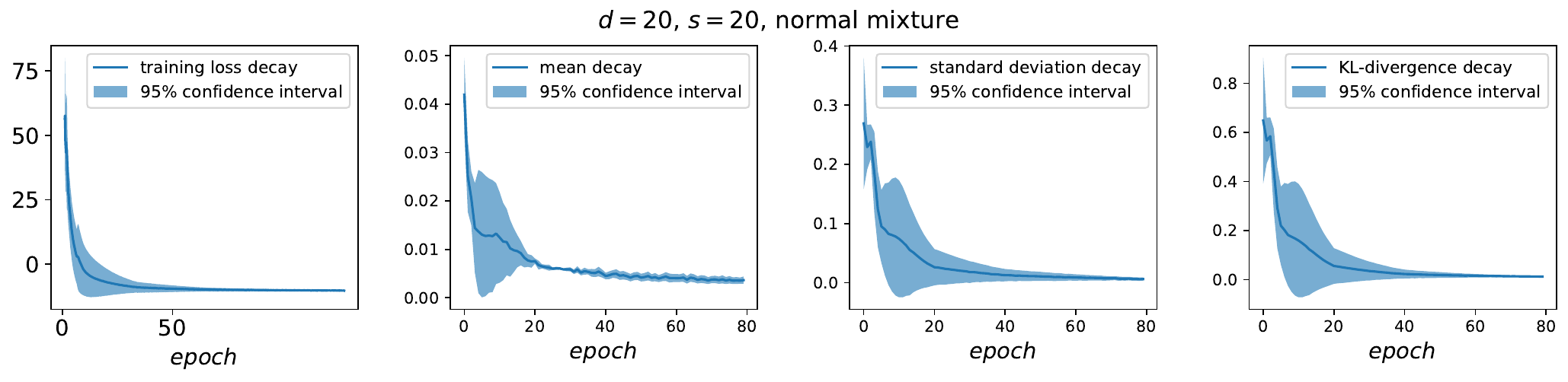}}
  {\includegraphics[width=0.95\textwidth]{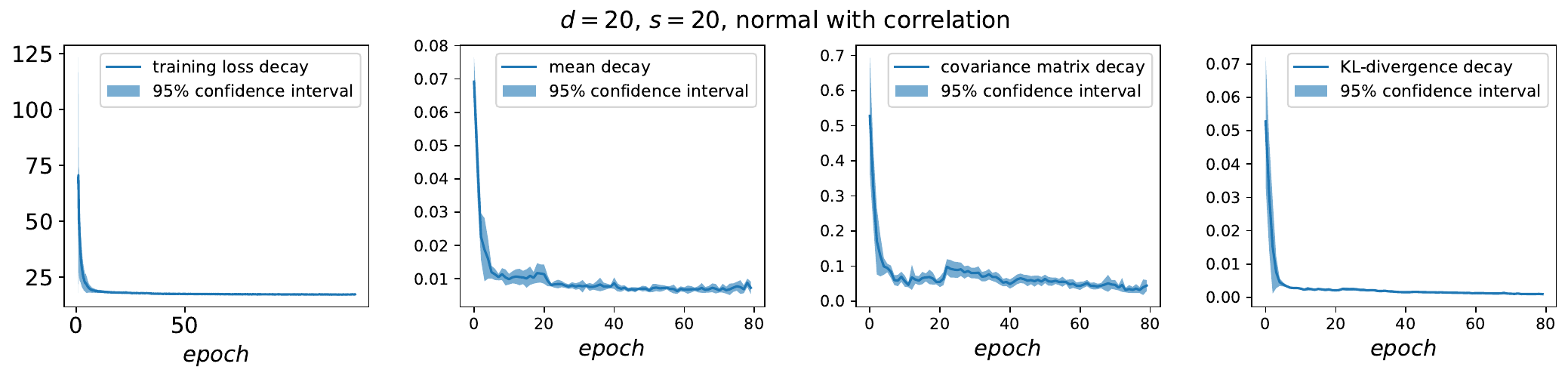}}
    \caption{
    The decay of training loss and average metrics defined by Eqs.~\eqref{eq_avgmean}-\eqref{eq_avgKL} are displayed alongside 95\% confidence intervals, based on ten simulations of randomly selected hyperparameters. The dimensionality of physical model is $d=20$, $s=20$, with three types of additive noise: uncorrelated multivariate normal distribution (top row), Gaussian mixture distribution with equal weight components (middle row), and correlated multivariate normal distribution (bottom row). The errors consistently converge to the minimum loss within a reasonable parameter range, showing stable trends after a brief transition period. These findings underscore the PR-NF model's effectiveness in addressing high-dimensional uncertainty challenges, highlighting its robustness and stability.}
    \label{fig_dim20}
\end{figure}

The optimal numerical results within the parameter set are recorded in Table \ref{tb_exhig}. As expected, the PR-NF model effectively evaluates the distribution of the ground truth across various types of additive noise. Notably, even with increasing dimensionality, the PR-NF model consistently maintains a high level of performance in approximation.

\begin{table}[h!]
\footnotesize
\renewcommand{\arraystretch}{1}
\centering
 \caption{The table shows the optimal numerical results of three average metrics derived from the simulations of Figs.~\ref{fig_dim5}, \ref{fig_dim10}, and \ref{fig_dim20}. Even with increasing dimensionality, the PR-NF
model consistently maintains a high level of performance in approximation. The remarkable accuracy performance of the PR-NF model is demonstrated.}
 \label{tb_exhig}
 \vspace{-0.2cm}
  \begin{tabular}{c c c c c c c c c}
  \hline 
  \multicolumn{9}{c}{$d = 20$, $s = 5$} \\
  \hline
 \multicolumn{3}{c}{normal without correlation} & \multicolumn{3}{c}{normal mixture}&
 \multicolumn{3}{c}{normal with correlation}\\
  \hline
${\rm Err_{mean}}$ &  ${\rm Err_{std}}$ & ${\rm Avg_{KL}}$ & ${\rm Err_{mean}}$ &  ${\rm Err_{std}}$ & ${\rm Avg_{KL}}$ & ${\rm Err_{mean}}$ &  ${\rm Err_{cov}}$ & ${\rm Avg_{KL}}$ \\
  \hline
4.31e-03 &9.78e-03 & 2.51e-02 & 4.25e-03 & 9.00e-03 & 1.15e-02 & 9.00e-03 & 4.51e-02 & 1.49e-03\\
  \hline
  \hline 
  \multicolumn{9}{c}{$d = 20$, $s = 10$} \\
  \hline
 \multicolumn{3}{c}{normal without correlation} & \multicolumn{3}{c}{normal mixture}&
 \multicolumn{3}{c}{normal with correlation}\\
  \hline
${\rm Err_{mean}}$ &  ${\rm Err_{std}}$ & ${\rm Avg_{KL}}$ & ${\rm Err_{mean}}$ &  ${\rm Err_{std}}$ & ${\rm Avg_{KL}}$ & ${\rm Err_{mean}}$ &  ${\rm Err_{cov}}$ & ${\rm Avg_{KL}}$ \\
  \hline
2.36e-03 &6.55e-03 & 5.06e-03 & 2.75e-03 & 4.75e-03 & 3.98e-03& 6.80e-03 & 3.93e-02 & 1.67e-03\\
  \hline
  \hline 
  \multicolumn{9}{c}{$d = 20$, $s = 20$} \\
  \hline
 \multicolumn{3}{c}{normal without correlation} & \multicolumn{3}{c}{normal mixture}&
 \multicolumn{3}{c}{normal with correlation}\\
  \hline
${\rm Err_{mean}}$ &  ${\rm Err_{std}}$ & ${\rm Avg_{KL}}$ & ${\rm Err_{mean}}$ &  ${\rm Err_{std}}$ & ${\rm Avg_{KL}}$ & ${\rm Err_{mean}}$ &  ${\rm Err_{cov}}$ & ${\rm Avg_{KL}}$ \\
  \hline
 2.42e-03& 3.06e-03& 7.04e-04 & 2.76e-03 & 7.46e-03 & 1.34e-03 & 7.99e-03 & 3.26e-02 & 8.79e-04\\
  \hline
  \end{tabular}
\end{table}

\subsection{Application to a geologic carbon storage problem}\label{gcs}

In addressing the greenhouse effect caused by trillions of tons of CO$_2$ in the atmosphere since the industrial age, it becomes evident that merely reducing emissions from the industrial, power, and transportation sectors is insufficient to combat climate change. Therefore, alongside deploying clean energy technologies for a decarbonized future, there is a critical need for carbon dioxide removal approaches \cite{alcalde2018estimating}. One promising option is the sequestration of CO$_2$ in geological reservoirs beneath the Earth's surface. This methodology involves capturing CO$_2$ emissions originating from stationary anthropogenic sources, which encompass fossil-fueled power plants and various industrial processes, followed by secure storage in diverse geologic formations, which may comprise saline-bearing formations, depleted oil and gas reservoirs, unmineable coal seams, and organic-rich shales \cite{fan2024advancing}. 

In subsurface storage, CO$_2$ is injected in its supercritical phase to depths where temperature and pressure conditions maintain the CO$_2$ in this phase, thereby optimizing storage volume utilization within reservoir pore spaces \cite{gholami2021leakage}. To ensure the safe and effective deployment of large-scale geologic carbon storage, a comprehensive risk management strategy is required to minimize and mitigate potential risks during CO$_2$ injection phases. A key aspect of this risk management approach is the continuous monitoring of  pressure fields over time to secure the containment of CO$_2$ within the reservoir \cite{fan2023deep}. For example, pressure buildup can pose a risk of leakage, especially when it reaches the fracture initiation pressure, potentially inducing leakage paths in caprock or faulting seals. These paths may serve as conduits for CO$_2$ migration from the designated storage location, potentially contaminating underground drinking water resources or escaping into the atmosphere \cite{chen2018geologic}. Conventionally, obtaining a reliable estimation of 3D pressure distributions over time in the subsurface relies on the inverse modeling process. Firstly, a forward reservoir model is developed and then calibrated using observational data through inverse modeling. Subsequently, this calibrated model is utilized to predict pressure distribution. However, this approach for real-time forecasting encounters challenges due to the computational expense of inverse modeling and the unreliability of predictions stemming from limited observations.


\begin{figure}[h!]
    \centering
  {\includegraphics[width=0.95\textwidth]{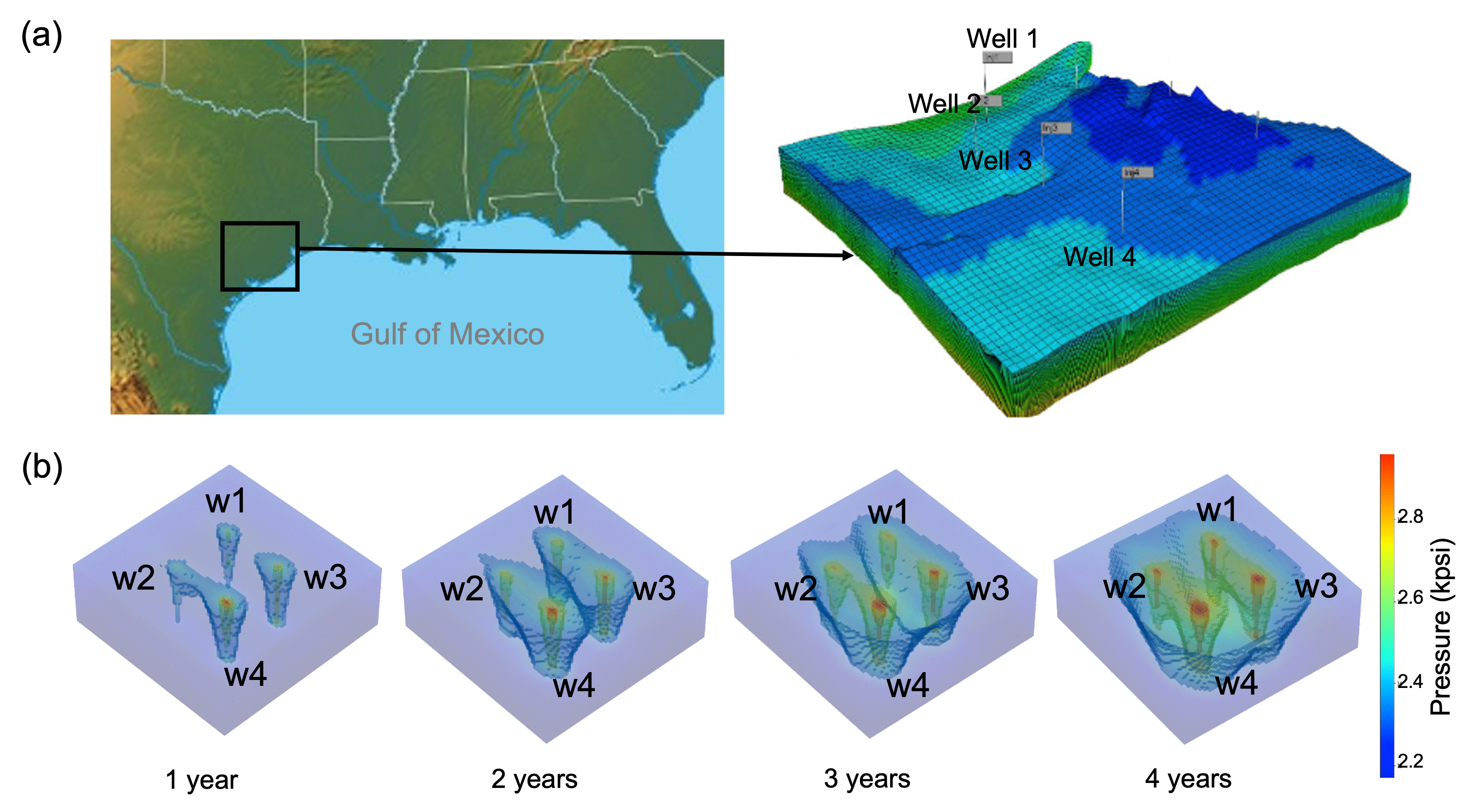}}
  \vspace{-0.3cm}
    \caption{Geologic carbon storage involves sequestering carbon dioxide (CO$_2$) underground in geological formations, aiming to mitigate the rising levels of CO$_2$ in the atmosphere. Monitoring pressure fields during the injection process is crucial for detecting potential leakage, which could contaminate underground drinking water or lead to CO$_2$ escaping into the atmosphere.
(a) In this work, we developed a physics-based reservoir simulation model adapted to a fluvial depositional environment, specifically focusing on the Gulf of Mexico High Island 24L. The model is designed to simulate the migration of CO$_2$ within subsurface formations. Four vertical injection wells were utilized to achieve the cumulative CO$_2$ injection target, with a constant injection rate over 30 years. 
(b) During the injection phase, the pressure buildup in the reservoir changes the state of stress, potentially increasing the risk of leakage. The objective is to forecast the entire 3D pressure fields over the injection period based on sparse observations from four injection wells combined with reservoir simulations. 
}
    \label{GCS_workflow}
\end{figure}

To generate the training dataset, we developed a physics-based reservoir simulation model representing a fluvial depositional environment, specifically the Gulf of Mexico High Island 24L, to simulate the migration of CO$_2$ in subsurface formations, as illustrated in Figure \ref{GCS_workflow}a. The geological model is structured with grid cells $54\times 48 \times 54$ along the $x$, $y$, and $z$ axes, with cell dimensions of 750 feet in the horizontal plane and 10 feet vertically. Four vertical injection wells are strategically positioned to attain cumulative CO$_2$ storage. Additionally, three production wells are positioned to prevent reservoir over-pressurization, maintaining a constant bottom-hole pressure constraint. We utilized CMG-GEM for reservoir simulations \cite{CMG}, generating 108 simulations collected every two months over 30 years, each consisting of 180 time steps. The input parameters include geological model realizations, cumulative CO$_2$ injection targets, and injection allocations for well pairs, while the output comprises pressure distributions.

In this study, we employ the developed PR-NF method to directly forecast the 3D pressure fields over time using observation data obtained from the injection wells, thereby avoiding the need for computationally demanding inverse modeling, as shown in Figure \ref{GCS_workflow}b. To implement this approach, we initially extract observation data from four injection wells, including every other layer between 33 and 51, resulting in a total of 40 observation variables. Subsequently, we employ a dimension reduction technique by utilizing a convolutional autoencoder to map the high-dimensional 3D pressure fields into a lower-dimensional latent space with 20 dimensions. Next, we forecast the 3D pressure fields in the latent space at the time step $t+1$ based on the observations at time step $t$. Finally, the predicted latent variables are transformed back to their original 3D space.

Figure \ref{GCS_pf} illustrates the forecasted pressure fields across layers 20, 30, and 40 at the end of the injection process for a sample in the testing set. The black triangles represent the four injection well locations. The first column represents the ground truths obtained through reservoir simulations, the second column illustrates the predictions generated by the developed PR-NF method, and the final column describes the error between these two sets of data. The comparative analysis between the synthetic truths and prediction results highlights minimal disparities in the pressure distribution. This observation demonstrates the PR-NF method's adeptness in capturing the spatial evolution characteristics of pressure fields, affirming its reliability as an effective monitoring tool for CO$_2$ storage operations.

\begin{figure}[h!]
    \centering
  {\includegraphics[width=0.95\textwidth]{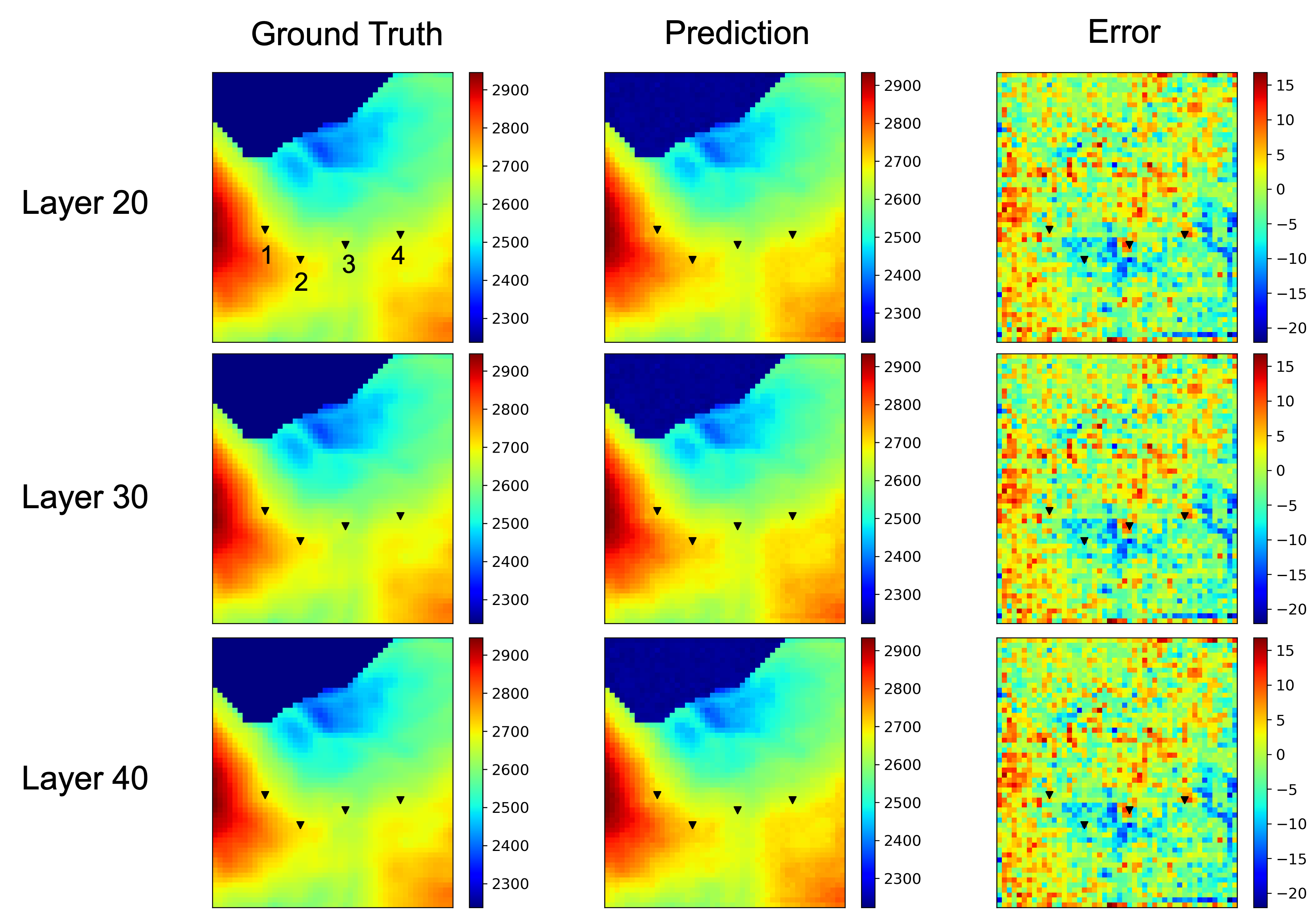}}
  \vspace{-0.3cm}
    \caption{Predictions of pressure fields across three layers at the termination of the injection process. The smaller errors between the ground truth and prediction demonstrates the effectiveness of the PR-NF approach in accurately forecasting the spatial evolution of pressure fields.  }
    \label{GCS_pf}
\end{figure}

To further explore the predictive capabilities of the PF-NF method, we analyze the temporal pressure distribution at four injection wells. Figure \ref{well_pressure} presents a comparison between the predicted pressure  and the actual pressure values over the injection process, using the same testing sample illustrated in Figure \ref{GCS_pf}. The remarkable consistency observed in the pressure predictions at each well location confirms the PF-NF method's ability to accurately forecast temporal pressure distribution during CO$_2$ injection. Notably, our proposed method not only enhances forecasting accuracy but also significantly expedites the forecasting process. For instance, the CMG reservoir simulator necessitates approximately 4.5 hours to complete a simulation, but the well-trained PR-NF model requires minutes to forecast 3D pressure fields. This accelerated forecasting capability represents a valuable enhancement to conventional forecasting workflow.

\begin{figure}[h!]
    \centering
  {\includegraphics[width=0.95\textwidth]{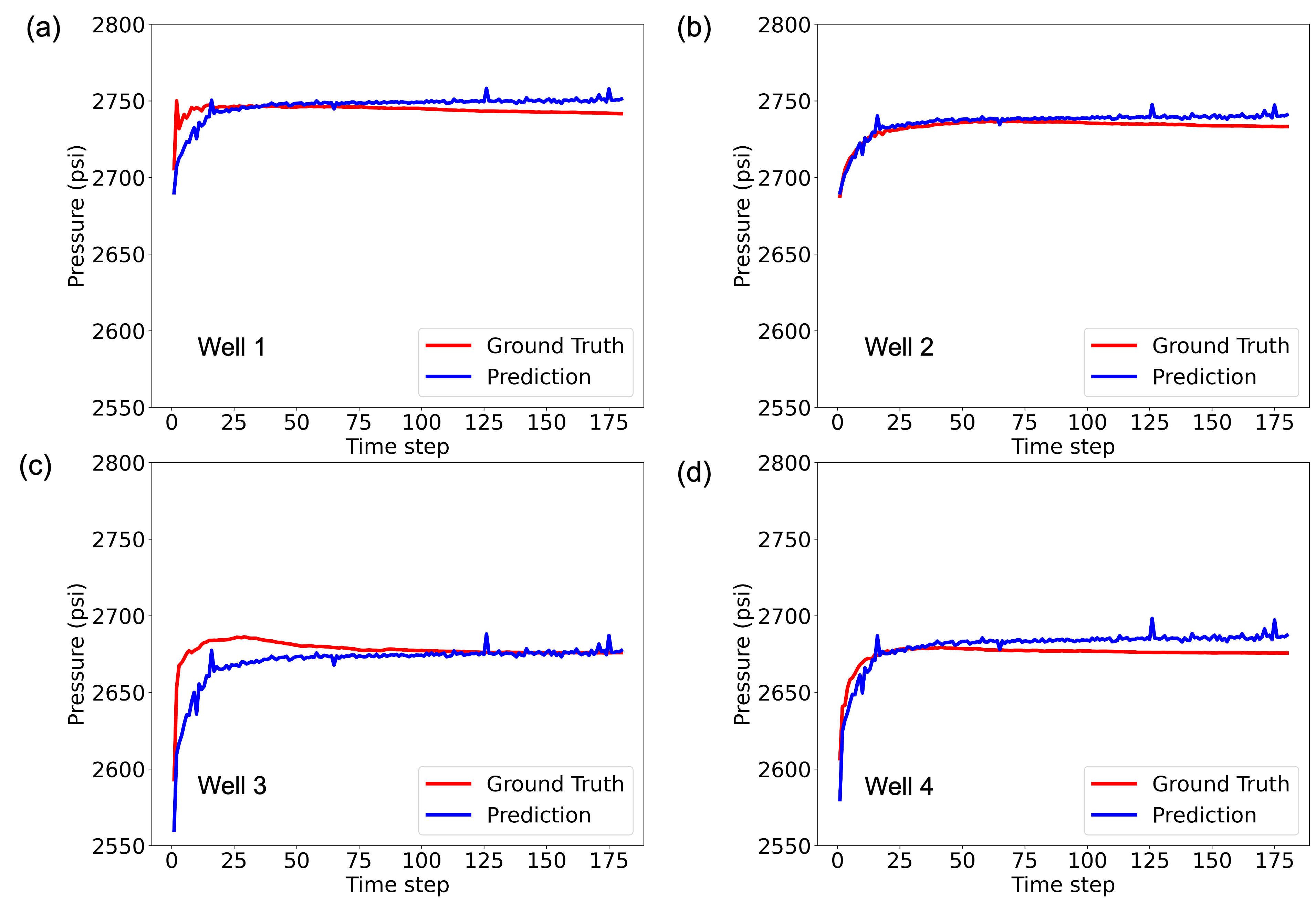}}
  \vspace{-0.3cm}
    \caption{Predictions of pressure over time at four injection wells for layer 30. The accurate predictions observed at each well effectively demonstrate the effectiveness of the PR-NF approach in forecasting pressure dynamics during the injection process. The ability to capture time-dependent features establishes it as a reliable tool for safer and more efficient CO$_2$ storage applications.}
    \label{well_pressure}
\end{figure}




\section{Conclusion}\label{sec:con}

In conclusion, our study introduces the conditional pseudo-reversible normalizing flow (PR-NF) model for surrogate modeling in quantifying uncertainty propagation. Through rigorous convergence analysis and three numerical experiments, we have demonstrated the efficacy and versatility of our approach.
Numerical experiments, including a synthetic example verifying algorithm accuracy, exploration of high-dimensional uncertainty problems, and application to a real-world geologic carbon storage challenge, collectively underscore the robustness and applicability of the PR-NF model. By directly learning and efficiently generating samples from conditional probability density functions, our model simplifies the modeling process while maintaining accuracy, without requiring prior knowledge of error terms or underlying physics functions.
In summary, our work contributes to advancing surrogate modeling techniques by offering a data-driven approach that enhances efficiency and accuracy in uncertainty quantification.

\section*{Acknowledgement}

This material is based upon work supported by the U.S. Department of Energy, Office of Science, Office of Advanced Scientific Computing Research, Applied Mathematics program under the contract ERKJ387, Office of Fusion Energy Science, and Scientific Discovery through Advanced Computing (SciDAC) program, at the Oak Ridge National Laboratory, which is operated by UT-Battelle, LLC, for the U.S. Department of Energy under Contract DE-AC05-00OR22725.
We thank Hongsheng Wang and Seyyed A. Hosseini for processing reservoir simulation data for studying the geologic carbon storage problem.

\bibliographystyle{abbrv}
\bibliography{UQ_ref}

\begin{thebibliography}{10}

\bibitem{abdar2021review}
M.~Abdar, F.~Pourpanah, S.~Hussain, D.~Rezazadegan, L.~Liu, M.~Ghavamzadeh,
  P.~Fieguth, X.~Cao, A.~Khosravi, U.~R. Acharya, et~al.
\newblock A review of uncertainty quantification in deep learning: Techniques,
  applications and challenges.
\newblock {\em Information fusion}, 76:243--297, 2021.

\bibitem{alcalde2018estimating}
J.~Alcalde, S.~Flude, M.~Wilkinson, G.~Johnson, K.~Edlmann, C.~E. Bond,
  V.~Scott, S.~M. Gilfillan, X.~Ogaya, and R.~S. Haszeldine.
\newblock Estimating geological co2 storage security to deliver on climate
  mitigation.
\newblock {\em Nature communications}, 9(1):2201, 2018.

\bibitem{blei2017variational}
D.~M. Blei, A.~Kucukelbir, and J.~D. McAuliffe.
\newblock Variational inference: A review for statisticians.
\newblock {\em Journal of the American statistical Association},
  112(518):859--877, 2017.

\bibitem{bohm2019uncertainty}
V.~B{\"o}hm, F.~Lanusse, and U.~Seljak.
\newblock Uncertainty quantification with generative models.
\newblock {\em arXiv preprint arXiv:1910.10046}, 2019.

\bibitem{chen2018geologic}
B.~Chen, D.~R. Harp, Y.~Lin, E.~H. Keating, and R.~J. Pawar.
\newblock Geologic co2 sequestration monitoring design: A machine learning and
  uncertainty quantification based approach.
\newblock {\em Applied energy}, 225:332--345, 2018.

\bibitem{CMG}
CMG.
\newblock Cmg gem user's guide.
\newblock {\em Calagry, Alberta, Canada:Computer Modeling Group Ltd.}, 2022.

\bibitem{dinh2016density}
L.~Dinh, J.~Sohl-Dickstein, and S.~Bengio.
\newblock Density estimation using real nvp.
\newblock {\em arXiv preprint arXiv:1605.08803}, 2016.

\bibitem{fan2023deep}
M.~Fan, D.~Lu, and S.~Liu.
\newblock A deep learning-based direct forecasting of co2 plume migration.
\newblock {\em Geoenergy Science and Engineering}, 221:211363, 2023.

\bibitem{fan2024advancing}
M.~Fan, H.~Wang, J.~Zhang, S.~A. Hosseini, and D.~Lu.
\newblock Advancing spatiotemporal forecasts of co2 plume migration using deep
  learning networks with transfer learning and interpretation analysis.
\newblock {\em International Journal of Greenhouse Gas Control}, 132:104061,
  2024.

\bibitem{gal2016dropout}
Y.~Gal and Z.~Ghahramani.
\newblock Dropout as a bayesian approximation: Representing model uncertainty
  in deep learning.
\newblock In {\em international conference on machine learning}, pages
  1050--1059. PMLR, 2016.

\bibitem{gholami2021leakage}
R.~Gholami, A.~Raza, and S.~Iglauer.
\newblock Leakage risk assessment of a co2 storage site: A review.
\newblock {\em Earth-Science Reviews}, 223:103849, 2021.

\bibitem{goodfellow2014generative}
I.~Goodfellow, J.~Pouget-Abadie, M.~Mirza, B.~Xu, D.~Warde-Farley, S.~Ozair,
  A.~Courville, and Y.~Bengio.
\newblock Generative adversarial nets.
\newblock {\em Advances in neural information processing systems}, 27, 2014.

\bibitem{guo2022normalizing}
L.~Guo, H.~Wu, and T.~Zhou.
\newblock Normalizing field flows: Solving forward and inverse stochastic
  differential equations using physics-informed flow models.
\newblock {\em Journal of Computational Physics}, 461:111202, 2022.

\bibitem{hoffman2013stochastic}
M.~D. Hoffman, D.~M. Blei, C.~Wang, and J.~Paisley.
\newblock Stochastic variational inference.
\newblock {\em Journal of Machine Learning Research}, 2013.

\bibitem{jospin2022hands}
L.~V. Jospin, H.~Laga, F.~Boussaid, W.~Buntine, and M.~Bennamoun.
\newblock Hands-on bayesian neural networks—a tutorial for deep learning
  users.
\newblock {\em IEEE Computational Intelligence Magazine}, 17(2):29--48, 2022.

\bibitem{karumuri2023learning}
S.~Karumuri and I.~Bilionis.
\newblock Learning to solve bayesian inverse problems: An amortized variational
  inference approach.
\newblock {\em arXiv preprint arXiv:2305.20004}, 2023.

\bibitem{kersaudy2015new}
P.~Kersaudy, B.~Sudret, N.~Varsier, O.~Picon, and J.~Wiart.
\newblock A new surrogate modeling technique combining kriging and polynomial
  chaos expansions--application to uncertainty analysis in computational
  dosimetry.
\newblock {\em Journal of Computational Physics}, 286:103--117, 2015.

\bibitem{kingma2013auto}
D.~P. Kingma and M.~Welling.
\newblock Auto-encoding variational bayes.
\newblock {\em arXiv preprint arXiv:1312.6114}, 2013.

\bibitem{padmanabha2021solving}
G.~A. Padmanabha and N.~Zabaras.
\newblock Solving inverse problems using conditional invertible neural
  networks.
\newblock {\em Journal of Computational Physics}, 433:110194, 2021.

\bibitem{papamakarios2021normalizing}
G.~Papamakarios, E.~Nalisnick, D.~J. Rezende, S.~Mohamed, and
  B.~Lakshminarayanan.
\newblock Normalizing flows for probabilistic modeling and inference.
\newblock {\em The Journal of Machine Learning Research}, 22(1):2617--2680,
  2021.

\bibitem{papamakarios2017masked}
G.~Papamakarios, T.~Pavlakou, and I.~Murray.
\newblock Masked autoregressive flow for density estimation.
\newblock {\em Advances in neural information processing systems}, 30, 2017.

\bibitem{psaros2023uncertainty}
A.~F. Psaros, X.~Meng, Z.~Zou, L.~Guo, and G.~E. Karniadakis.
\newblock Uncertainty quantification in scientific machine learning: Methods,
  metrics, and comparisons.
\newblock {\em Journal of Computational Physics}, 477:111902, 2023.

\bibitem{ratzlaff2019hypergan}
N.~Ratzlaff and L.~Fuxin.
\newblock Hypergan: A generative model for diverse, performant neural networks.
\newblock In {\em International Conference on Machine Learning}, pages
  5361--5369. PMLR, 2019.

\bibitem{selvan2020uncertainty}
R.~Selvan, F.~Faye, J.~Middleton, and A.~Pai.
\newblock Uncertainty quantification in medical image segmentation with
  normalizing flows.
\newblock In {\em Machine Learning in Medical Imaging: 11th International
  Workshop, MLMI 2020, Held in Conjunction with MICCAI 2020, Lima, Peru,
  October 4, 2020, Proceedings 11}, pages 80--90. Springer, 2020.

\bibitem{sensoy2020uncertainty}
M.~Sensoy, L.~Kaplan, F.~Cerutti, and M.~Saleki.
\newblock Uncertainty-aware deep classifiers using generative models.
\newblock In {\em Proceedings of the AAAI conference on artificial
  intelligence}, volume~34, pages 5620--5627, 2020.

\bibitem{wilson2020bayesian}
A.~G. Wilson and P.~Izmailov.
\newblock Bayesian deep learning and a probabilistic perspective of
  generalization.
\newblock {\em Advances in neural information processing systems},
  33:4697--4708, 2020.

\bibitem{wirtz2015surrogate}
D.~Wirtz, N.~Karajan, and B.~Haasdonk.
\newblock Surrogate modeling of multiscale models using kernel methods.
\newblock {\em International Journal for Numerical Methods in Engineering},
  101(1):1--28, 2015.

\bibitem{yang2022diffusion}
L.~Yang, Z.~Zhang, Y.~Song, S.~Hong, R.~Xu, Y.~Zhao, W.~Zhang, B.~Cui, and
  M.-H. Yang.
\newblock Diffusion models: A comprehensive survey of methods and applications.
\newblock {\em ACM Computing Surveys}, 2022.

\bibitem{yang2023pseudo}
M.~Yang, P.~Wang, D.~del Castillo-Negrete, Y.~Cao, and G.~Zhang.
\newblock A pseudo-reversible normalizing flow for stochastic dynamical systems
  with various initial distributions.
\newblock {\em arXiv preprint arXiv:2306.05580}, 2023.

\bibitem{zhao2017learning}
T.~Zhao, R.~Zhao, and M.~Eskenazi.
\newblock Learning discourse-level diversity for neural dialog models using
  conditional variational autoencoders.
\newblock {\em arXiv preprint arXiv:1703.10960}, 2017.

\bibitem{zhou2005study}
Z.~Zhou, Y.~S. Ong, M.~H. Nguyen, and D.~Lim.
\newblock A study on polynomial regression and gaussian process global
  surrogate model in hierarchical surrogate-assisted evolutionary algorithm.
\newblock In {\em 2005 IEEE congress on evolutionary computation}, volume~3,
  pages 2832--2839. IEEE, 2005.

\end{thebibliography}
\end{document}